\documentclass[prodmode,acmtist]{acmsmall} % Aptara syntax

\usepackage[ruled]{algorithm2e}
\renewcommand{\algorithmcfname}{ALGORITHM}
\SetAlFnt{\small}
\SetAlCapFnt{\small}
\SetAlCapNameFnt{\small}
\SetAlCapHSkip{0pt}
\IncMargin{-\parindent}

\usepackage{caption}
\usepackage{algorithmic}
\usepackage{makeidx}
\usepackage{graphicx}
\usepackage{amsmath}
\usepackage{color}
\usepackage{multirow}

\acmVolume{00}
\acmNumber{00}
\acmArticle{00}
\acmYear{2015}
\acmMonth{00}

\newtheorem{observation}{Observation}

\newcommand{\prob}{\operatorname{prob}}

\newcommand{\cov}{\operatorname{cov}}
\newcommand{\supp}{\operatorname{supp}}

\newcommand{\od}{\operatorname{oddsratio}}

\begin{document}

% Page heads
\markboth{J. Li et al.}{From Observational Studies to Causal Rule Mining}

% Title portion
\title{From Observational Studies to Causal Rule Mining}
\author{JIUYONG LI
\affil{University of South Australia}
THUC DUY LE
\affil{University of South Australia}
LIN LIU
\affil{University of South Australia}
JIXUE LIU
\affil{University of South Australia}
ZHOU JIN
\affil{University of Science and Technology China}
BINGYU SUN
\affil{Chinese Academy of Sciences}
SAISAI MA
\affil{University of South Australia}}

% NOTE! Affiliations placed here should be for the institution where the
%       BULK of the research was done. If the author has gone to a new
%       institution, before publication, the (above) affiliation should NOT be changed.
%       The authors 'current' address may be given in the "Author's addresses:" block (below).
%       So for example, Mr. Abdelzaher, the bulk of the research was done at UIUC, and he is
%       currently affiliated with NASA.

\begin{abstract}
Randomised controlled trials (RCTs) are the most effective approach to causal discovery, but in many circumstances it is impossible to conduct RCTs. Therefore observational studies based on passively observed data are widely accepted as an alternative to RCTs. However, in observational studies, prior knowledge is required to generate the hypotheses about the cause-effect relationships to be tested, hence they can only be applied to problems with available domain knowledge and a handful of variables. In practice, many data sets are of high dimensionality, which leaves observational studies out of the opportunities for causal discovery from such a  wealth of data sources. In another direction, many efficient data mining methods have been developed to identify associations among variables in large data sets. The problem is, causal relationships imply associations, but the reverse is not always true. However we can see the synergy between the two paradigms here. Specifically, association rule mining can be used to deal with the high-dimensionality problem while observational studies can be utilised to eliminate non-causal associations. In this paper we propose the concept of causal rules (CRs) and develop an algorithm for mining CRs in large data sets. We use the idea of retrospective cohort studies to detect CRs based on the results of association rule mining. Experiments with  both synthetic and real world data sets have demonstrated the effectiveness and efficiency of CR mining. In comparison with the commonly used causal discovery methods, the proposed approach in general is faster  and has better or competitive performance in finding correct or sensible causes. It is also capable of finding a cause consisting of multiple variables, a feature that other causal discovery methods do not possess.
\end{abstract}


\acmformat{Jiuyong Li, Thuc Duy Le, Lin Liu, Jixue Liu, Zhou Jin, Bingyu Sun, and Saisai Ma, 2015. From Observational Studies to Causal Rule Mining.}
% At a minimum you need to supply the author names, year and a title.
% IMPORTANT:
% Full first names whenever they are known, surname last, followed by a period.
% In the case of two authors, 'and' is placed between them.
% In the case of three or more authors, the serial comma is used, that is, all author names
% except the last one but including the penultimate author's name are followed by a comma,
% and then 'and' is placed before the final author's name.
% If only first and middle initials are known, then each initial
% is followed by a period and they are separated by a space.
% The remaining information (journal title, volume, article number, date, etc.) is 'auto-generated'.

\begin{bottomstuff}
A preliminary version of this work was published in the Proceedings of 2013 IEEE 13th International Conference on Data Mining Workshops, the First IEEE ICDM Workshop on Causal Discovery 2013 (CD2013), pp. 114-123, Dallas, Texas, USA, December 7-10, 2013.

Authors' addresses: J. Li, T. D. Le, L. Liu J. Liu {and} S. Ma, School of Information Technology and Mathematical Sciences, University of South Australia, Mawson Lakes, SA, 5095, Australia; Z. Jin, Department of Automation, University of Science and Technology, Hefei 230026, China; B. Sun, Institute of Intelligent Machines, Chinese Academy of Sciences, Hefei 230031, China.

Permission to make digital or hard copies of part or all of this work for personal or classroom use is granted without fee provided that copies are not made or distributed for profit or commercial advantage and that copies bear this notice and the full citation on the first page. Copyrights for third-party components of this work must be honored. For all other uses, contact the Owner/Author.

Copyright 2015 held by Owner/Author

Publication title/2157-6904/2015/MonthOfPublication - ArticleNumber

http://dx.doi.org/10.1145/2746410
\end{bottomstuff}

\maketitle

\section{Introduction}\label{sectionIntro}

%What is causal discovery -- the need/significance of causal discovery
Causal discovery aims to infer the cause-effect relationships between variables. Such relationships imply the mechanism of outcome variables taking their values and how the change of cause variables would lead to the change of the outcome variables \cite{Spirtes2010}. In other words, causality provides the basis for explaining how things have happened and for predicting how the outcomes would be when their causes have changed. Therefore apart from being a fundamental philosophical topic, causality has been studied and utilised in almost all disciplines, e.g. medicine, epidemiology, biology, economics, physics, social science, as a basic and effective tool for explanation, prediction and decision making \cite{Guyon2010,Kleinberg2011}. Some specific examples include the applications in medicine for developing new treatments or drugs for a disease; and in economics for forecasting the results of a particular financial policy and in turn to assist decision and/or policy making.

%Randomised controlled trials -- problems
Randomised controlled trials (RCTs) are recognised as the gold standard for testing the effects of interventions \cite{Shadish2002,Stolberg2004}. However, it is also widely acknowledged that in many cases it is impossible to conduct RCTs due to cost and/or ethical concerns. For example, to find out the causal effect of alcohol consumption on heart diseases, it will be unethical to require an experiment participant to drink. Sometimes it is totally forbidden to manipulate a possible cause factor, for example, in a life-threatening situation. %Furthermore, as the number of causes increases RCTs become increasingly difficult and costly to apply, particularly in the case of high dimensional problem of a large number of potential causes.

%Observational studies -- widely used, effectiveness, why (e.g. cohort study, matching distributions of covariates in two groups)
Under these circumstances, observational studies \cite{Rosenbaum2010,Concato2000} are considered as the best alternatives to RCTs, and it has been shown that well-designed observational studies can achieve comparative results as RCTs \cite{Concato2000}. As suggested by the name, observational studies are based on passively observed data, and they do not require manipulation of an exposure (i.e. a potential cause factor). There are two main types of observational studies for causal discovery, cohort studies and case-control studies \cite{Song2010,Blackmore2004,Euser2009}. In a cohort study, based on the status of being exposed to a potential cause factor (e.g. certain radiation), an exposure group of subjects and a non-exposure or control group of subjects are selected, and then followed to observe the occurrence of the outcome (e.g. cancer). In a case-control study, subjects are selected based on the status of the outcome, i.e. the case group consisting of subjects with the outcome and a control group of subjects without the outcome are identified, and then their status of exposure to the potential cause factor is examined. In both types of studies, the effect of an exposure on the outcome is determined by comparing the difference between the exposed/case group and control group. In order to achieve convincing result, an observational study must try to replicate a RCT as much as possible, i.e. the covariate distributions of the two contrasting groups should be as close as possible.
%***One way of achieving this goal is through matching. --- talk about this later and include references
%***timeline - exposure occurs, then outcome. Did we pay attention to this?

%problems with high-dimensional data: automated hypotheis generation & compound cause factors - to utilise the wealth information hidden in observational data for causal exploration, desirable to have a techinique for automatic hypotheis generation
Although observational studies provide an effective approach to causal discovery, they work in the fashion of hypothesis testing, that is, at the commencement of a study, a cause-effect relationship needs to be hypothesised. Then data are collected or retrieved from databases for testing the hypothesis. This requires the prior knowledge or anticipation of the exposures and outcomes, which may not always be available, especially when the number of variables under study is large and the purpose is to explore possible cause-effect relationships, instead of validating an assumed causal relationship. For example, in the study of gene regulation, we may have a clear idea of the possible genetic diseases (outcomes), but which genes could be the possible genetic causes of the diseases may not be known at all. Given the huge number of genes (tens of thousands), it is infeasible to test each gene to find the causes. Therefore to exploit the wealth information in observational data using the well-established methodology of observational studies, we firstly need some efficient ways to generate the hypotheses with high confidence.

Another challenge with observational studies (as well as RCTs) is that even with domain knowledge, it is difficult to foresee a combined cause. For example, multiple genes may work together to cause a disease, which is normally hard to identify with domain knowledge only.

%In the area of data mining, association analysis, i.e. association rule mining -- effcient -- rules with mutiple LHS - suitable for the above requirements when a target is given or not given. -- spurious associations, e.g. simpson paradox -- can be fixed by the methodology of observational study ==> synergy of the two paradig
This is where we can take the advantage of the outcome of data mining research. In the last two decades, huge efforts have been made on association rule mining \cite{Agrawal93} and many efficient algorithms have been developed to discover association rules from large data sets \cite{DataMiningBookHan2005}. An association rule represents interesting associations among variables, for example,  $pizza \rightarrow garlic\ bread$; $\{strong\ wind,\ high\ temperature\} \rightarrow falling\ trees$. Although statistical associations do not necessarily mean causality (for instance, buying garlic bread and pizza together does not indicate that buying one is the cause of buying the other. Mostly likely this is a consequence of a meal deal), it is commonly accepted that associations are necessary for causality.

Our idea is thus to utilise the synergy of observational studies and association rule mining to develop an efficient method for \emph{automated} discovery of  causal relationships in large data sets. We firstly use association rule mining to find out the hypothesised cause-effect relationships (represented as association rules) regarding an outcome. Then for each of the hypotheses, we conduct an observational study (e.g. a cohort study) to test if the exposure is a real cause, i.e. to identify if the association rule is a \emph{causal rule}.

As the LHS (left-hand-side or antecedent) of an association rule can comprise multiple attributes, a favourable consequence of using association rule mining here is that it can generate hypothesised causal relationships with compound exposure, such as the rule shown above $\{strong\ wind,\ high\ temperature\} \rightarrow falling\ trees$. In this case, we consider the two attributes as one variable/exposure in our observational studies, hence the validity of the combined cause can be tested.

In the rest of the paper, we will present the definition of causal rules and our approach to identifying CRs (Section 3), the algorithm for mining CRs (Section 4), and the experiment results demonstrating the effectiveness and efficiency of the algorithm (Section 5). Before the presentation, in Section 2, we firstly outline the related work and show the contribution of this paper.
%We propose a method ...

%Current progress in automated/computational causal discovery -- BN learning -- problems: high-dimensional as a complete network is to be learnt, but in practice, fixed target problem often arise -- local discovery, more efficient, for a given target
\section{Related Work and Contribution}\label{sectionRelatedWork}
Observational studies \cite{Rosenbaum2010,Concato2000} have had a very long history, and there has been a great deal of research on observational studies, by both statisticians and practitioners in medicine and other application areas. The main focus of the research is on how to design good observational studies, including selection of subjects or records, methods for identifying exposed and non-exposed groups to replicate RCTs as closely as possible, and the ways for analysing the data. However, as far as we know, there is little work done on using observational studies for automated causal discovery in large, especially high-dimensional data.

In the field of computer science, causal discovery from observational data has attracted enormous research efforts in the past three decades. Currently Bayesian network techniques are at the core of the methodologies for causal discovery in computer science \cite{Spirtes2010}. Bayesian networks provide a graphical representation of conditional independence among a set of variables. Under certain causal assumptions, a directed edge between two nodes (variables) in a Bayesian network represents a causal relationship between the two variables \cite{Spirtes2010,Spirtes2001book}. Over the years, many algorithms have been developed for learning Bayesian networks from data \cite{Neapolitan2003,Spirtes2010}. However, up to now it is only feasible to learn a Bayesian network with dozens of variables, or hundreds if the network is sparse \cite{Spirtes2010}. Therefore, in practice it is infeasible to identify causal relationships using Bayesian network based approaches in most cases.

Indeed the difficulties faced by these causal discovery approaches originate from their goal, i.e. to discover a complete causal model of the domain under consideration. Such a model indicates all pairwise causal relationships among the variables. This, unfortunately, is essentially impossible to achieve when the domain contains a large number of variables. It has been shown that in general learning a Bayesian network is NP-hard \cite{Chickering2004}.

Some constraint based approaches do not search for a complete Bayesian network, so they can be more efficient for causal relationship discovery. Several such algorithms have shown promising results \cite{Cooper:LCD1997,Silverstein:CCU2000,Mani:Y2006,Pellet:Blanket2008,Aliferis2010}. Based on observational data, these methods determine conditional independence of variables and learn local causal structures. However some of the methods are only capable of discovering the causal relationships represented with some fixed structures, e.g. CCC \cite{Cooper:LCD1997}, CCU \cite{Silverstein:CCU2000} and the Y structures \cite{Mani:Y2006}, and they do not identify causal relationships that cannot be represented with these structures.  The complexity of other methods for learning a partial Bayesian network in general is still exponential to the number of variables, unless accuracy and/or completeness are traded with efficiency \cite{Aliferis2010}.

%We are tackling this computational problem from a different avenue, i.e. by taking the synergy of obervational studies and data mining, so our contribution is not only effective and efficient method for causal discovery from large datasets with a given target, but also push the two fields by: enable automated cohort studies with large data set; improve associate rule mining

%Our method tackles the problem of causal discovery from a different perspective. As discussed above, our approach is to integrate two well-established methodologies for relationship discoveries. Consequently, while the main contribution of this paper is a statistically sound and computational efficient causal discovery method, our work also contributes to each of the two areas: on one hand, associate rule mining facilitates automated hypothesis generation, enabling the application of observational studies to high-dimensional data and for finding combined causes; on the other hand, observational studies allow the elimination of non-causal rules, which provides an effective solution to the major hurdle for association rule applications \cite{Webb2008,Webb2009,Tan2004,Lenca2008}.

Our method tackles the problem of causal discovery from a different perspective. It integrates two well-established methodologies in two different fields for relationship discoveries. The main contribution of this paper is to propose a statistically sound and computational efficient causal discovery method for causal relationship exploration. Cohort studies have been widely accepted for identifying causal links in health, medical and social studies, so the use of cohort studies to uncover causal relationships is methodologically sound. In this paper, the theoretical validity of the proposed method has also been justified by its connection with a well-known causal inference framework -- the potential outcome model \cite{Pearl2000,Morgan2007}.  Our goal is to automate causal relationship discovery in data, making it possible to explore causal relationships in both large and high dimensional data sets.

 Our work also contributes to the area of association rule mining. Association rule mining is a main data mining technique and has many applications in various fields, but a major obstacle of association rule mining  is that it produces too many rules and many of them are uninteresting since they represent random associations in a data set \cite{Webb2008,Webb2009,Tan2004,Lenca2008}. Cohort studies enable us to filter out a large proportion of such uninteresting rules and keep the most interesting ones for a broad range of applications since discovering causal relationships is the goal of the majority of applications.

 This paper is an extension of our preliminary work in~\cite{LiCR2013}, with three major developments: (1) A more explicit presentation of the motivation, goal and contribution of the research in the newly written Sections 1 and 2; (2) A new section (Section \ref{sectionPotentialOucome}) for justifying the validity of the CR framework; (3) A new set of experiments with a total of 13 synthetic data sets for evaluating the performance and scalability of the proposed method , and new experiments on investigating the effect of different matching methods (see Section 5).

\section{Causal Rules}

\subsection{Notations}

Let $D$ be a data set for a set of binary variables $(X_1, X_2, \dots, X_m, Z)$, where $X_1, X_2, \dots, X_m$ are \emph{predictor} variables and $Z$ is a \emph{response} variable. Values of $Z$ are of user's interest, e.g. having a disease or being normal. Considering a binary data set makes the conceptual discussions in the paper easier, and it does not lose the generality of a data set that contains attributes of multiple discrete values. For example, a multi-valued data set for the variables (Gender, Age, $\dots$) is equivalent to a binary data set for the variables (Male, Female, 0-19, 20-39, 40-69, $\dots$). In this paper, both the Male and Female variables are kept to allow us to have combined variables that involve them separately, for example, (Female, 40-59, Diabetes) and (Male, 40-59, Smoking).

$P$ is a \emph{combined} variable if it consists of multiple variables $X_1, \dots, X_n$ where $n \ge 2$, and $P=1$ when $(X_1 = 1,  \dots, X_n=1)$ and $P=0$ otherwise.

A \emph{rule} is in the form of $(P=1) \to (Z=1)$, or $p \to z$ where $z$ stands for $Z=1$ and $p$ for $P=1$. $p$ is also called a $\emph{k-pattern}$ where $k$ is the length of $P$ (the number of component variables of $P$). Our ultimate goal is to find out whether $p \to z$ is a causal rule.

\subsection{Association rules}\label{subsectionAR}

With our approach, we first consider the association between $P$ and $Z$ since an association is normally necessary for a causal relationship.

Odds ratio is a widely used measure for associations in retrospective studies \cite{Fleiss2003}, and we define the odds ratio of a rule as follows.

\begin{definition}[Odds ratio of a rule] \label{def_oddsratio} Given the following contingency table of a rule, $p \to z$,
\begin{center}
\begin{tabular}{|c|c|c|}
\hline
 & $z (Z=1)$ & $\neg z (Z=0)$ \\
\hline $p (P=1)$ & $\supp(pz)$ & $\supp(p \neg z)$ \\
\hline $\neg p (P=0)$  & $\supp( \neg p z)$ & $\supp(\neg p \neg z)$ \\
\hline
\end{tabular}
\end{center}

\noindent where $\supp(x)$ indicates the support of pattern $X$, the count of value $x$ in the given data set, $D$, and we have $\supp (p) = \supp(pz) + \supp(p \neg z)$, $\supp (z) = \supp(pz) + \supp(\neg p z)$, and $\supp(pz) + \supp(p \neg z) + \supp(\neg p z) + \supp(\neg p \neg z) = n$, where $n$ is the number of records in the data set, then the odds ratio of the rule $p \to z$ on $D$ is defined as:
\begin{equation}\label{orationRule}
\od_{\textsc{d}} (p \to z) = \frac{\supp(pz)*\supp(\neg p \neg z)}{\supp(p \neg z)*\supp(\neg p z)}
\end{equation}
\end{definition}

From the definition, the odds ratio of a rule is the ratio of the odds of value $z$ occurring in group $P=1$ to the odds of value $z$ occurring in group $P =0$, so an odds ratio of 1 means that $z$ has an equal chance to occur in both groups, and an odds ratio deviating from 1 indicates an  association (positive or negative) between $Z$ and $P$.

\begin{definition}[Association rule] Using the notations in Definition~\ref{def_oddsratio}, the support of a rule $p \to z$ is defined as $\supp(p \to z) = \supp(pz)$. Given a data set $D$, let $min\_supp$ and $min\_oratio$ be the minimum support and odds ratio respectively, $p \to z$ is an association rule if $\supp(p \to z)>min\_supp$ and $\od_{\textsc{d}} (p \to z)>min\_oratio$, and LHS($p \to z$) = $p$ and RHS($p \to z$) = $z$.
\end{definition}

In the definition, we consider $z$ as the RHS of a rule. An association rule that has $\neg z$ ($Z=0$) as its RHS can be defined in the same way. These association rules ($p \to z$ and $p \to \neg z$) are class association rules~\cite{LiuHsuMa98} where the confidence ($\prob (z | p)$) is replaced by the odds ratio.  Furthermore, only positive association between a predictor variable and the response variable is considered in the above definition as in most cased in practice, we are concerned about the occurrence of the predictor (i.e. $P=1$) leading to the occurrence of the response (i.e. $Z=1$).

We note that the distribution of the values of the response variable can be skewed and a uniform minimum support may lead to too many rules for the frequent values and few rules for the infrequent values. In the implementation, we use the local support that is relative to the frequency of a value in the response variable, i.e. $\ lsupp(p \to z) = \frac{\supp(pz)}{supp(z)}$. The local support is a ratio and can be set the same, say 5\%, for rules that have $z$ or $\neg z$ as the RHS.

Traditional association rules are defined by support and confidence \cite{Agrawal93}. An association rule in the support and confidence scheme may not show a real association between the LHS and RHS of a rule \cite{Brin97}. Therefore in the above definition, we use odds ratio as the indicator of association. The minimum odds ratio in the definition may be replaced by a significance test on $\od_{\textsc{d}} (p \to z) > 1$ to ensure that an association rule indicates a significant association between the LHS and RHS of the rule.

The test of significant association is determined as the following.

Let $\omega$ be the odds ratio of the rule $p \to z$ on the given data set $D$, i.e. $\od_{\textsc{d}} (p \to z) = \omega$. The confidence interval of $\omega$, $[\omega_{-}, \omega_{+}]$, is defined as \cite{Fleiss2003}:\\

\[
\omega_-=\exp(\ln{\omega}-z' \sqrt{\frac{1}{\supp(pz)}+\frac{1}{\supp(p \neg z)}+\frac{1}{\supp(\neg p z)}+\frac{1}{\supp(\neg p \neg z)}})
\]
\noindent and
\[
\omega_+=\exp(\ln{\omega}+z' \sqrt{\frac{1}{\supp(pz)}+\frac{1}{\supp(p \neg z)}+\frac{1}{\supp(\neg p z)}+\frac{1}{\supp(\neg p \neg z)}})
\]

\noindent where $z'$ is the critical value corresponding to a desired level of confidence ($z'=1.96$ for 95\% confidence). $\omega_-$ and $\omega_+$ are  the lower and upper bounds respectively of an odds ratio at a confidence level. If $\omega_- > 1$, the odds ratio is significantly higher than 1, hence $P$ and $Z$ are associated. Equivalently, $p \to z$ is an association rule. %Therefore, we do not use the minimum odds ratio in the algorithm.

An important advantage of the above process is that it is automatically adaptive to the size of a data set. For a large data set, the confidence interval of an odds ratio is small and hence a small odds ratio can be significantly higher than 1. For a small data set, the confidence interval of an odds ratio is large and hence a large odds ratio is needed to be significantly higher than 1.

% { (see Section \ref{sectionCandidateRule} for details.)}

However, statistically reliable associations do not always indicate causal relationships although causality is mostly observed as associations in data, which can be illustrated by the following example.

\begin{example}
\label{ex_simpsionParadox}
Suppose that we have generated an association rule: ``Gender = $m$" $\to$ ``Salary = $low$" from a data set with the following statistics:
\begin{center}
\begin{tabular}{|c|c|c|}
\hline
 & Salary = $low$ & Salary = $high$ \\
\hline Gender = $m$ & 185 & 120 \\
\hline Gender = $f$  & 65 & 60 \\
\hline
\end{tabular}
\end{center}
The ratio of low salary earners to high salary earners in the male group is 1.54:1 while the ratio in the female group is 1.08:1. In other words, the odds for a male worker receiving a low salary is 1.54 and the odds for a female worker receiving a low salary is 1.08. The odds ratio of male and female groups receiving low salaries is 1.43, which is greater than 1. Therefore as described previously, this odds ratio indicates a positive association between ``Gender = $m$'' and ``Salary = $low$''.

Is this association valid? Let us do further analysis by stratifying the samples by the Education attribute. Assume that the statistics of the stratified data sets are:

\begin{center}
\begin{tabular}{|c|c|c|}
\hline
 & Salary = $low$ & Salary = $high$ \\
\hline Gender = $m$ \& College = $y$ & 5 & 20 \\
\hline Gender = $f$ \& College = $y$ & 15 & 40 \\
\hline
\end{tabular}
\end{center}
and

\begin{center}
\begin{tabular}{|c|c|c|}
\hline
 & Salary = $low$ & Salary = $high$ \\
\hline Gender = $m$ \& College = $n$ & 180 & 100 \\
\hline Gender = $f$  \& College = $n$ & 50 & 20 \\
\hline
\end{tabular}
\end{center}

The above two tables indicate a negative association between ``Gender = $m$'' and ``Salary = $low$'' because the odds ratio in the College education group is 0.67 and odds ratio in the non-College education group is 0.72. Both contradict the association rule ``Gender = $m$'' $\to$ ``Salary = $low$''.

We obtain two conflicting results here. This means that an association may be volatile in a sub data set or a super data set. This is a phenomenon of the famous Simpson Paradox~\cite{Pearl2000}, indicating that associations may not imply causal relationships.
\end{example}
Therefore our idea is to conduct a retrospective cohort study to detect true causal relationships from identified association rules.

\subsection{Cohort study}

As discussed in Section \ref{sectionIntro}, when randomised controlled trials are practically impossible, observational studies are often used as the alternative approach to finding out the possible cause-effect relationships. A major type of observational studies is cohort studies, which can be conducted in either of the two ways, prospective and retrospective \cite{Euser2009,Fleiss2003}. In a perspective cohort study, researchers follow cohorts over time to observe their development of a certain outcome. In a retrospective study, researchers look back at events that already occurred. In a data mining setting, as the data we have are historical records, we adopt the idea of a retrospective cohort study in this paper.

A retrospective cohort study selects individuals who have exposed and have not exposed to a suspected risk factor but are alike within many other aspects. For example, middle aged male who have been smoking and who have not been smoking for a certain time period are selected for studying the effect of smoking on lung cancer. Here smoking is the risk factor or \emph {exposure variable}, and ``middle aged'' and ``males'' indicate the common characteristics shared by the two cohorts. A significant difference in the value of the outcome or response variable (lung cancer) of the two cohorts indicates a possible causal relationship between the exposure variable and the response variable.

In the rest of the paper, with a binary exposure variable, we call the cohort where the exposure variable takes value 1 the \emph{exposure group}, the cohort where the exposure variable takes value 0 the \emph{non-exposure group}, and the set of variables determining the common characteristics of the two groups the \emph{control variable set}.

From the above description, the core requirement for a cohort study is to obtain the matched exposure and non-exposure groups such that the distribution of control variable set of the two groups are the same or very similar. For example, in a cohort study to test whether gender is a cause of salary difference, the exposure variable is gender and the control variable set consists of variables: education, profession, experience and location. From a given data set, we will need to select samples for the exposure and non-exposure groups so that the two groups have the same distribution regarding the control variables. Then if there is a significant difference in salary between the two groups, we can conclude that gender is a cause of salary difference.

In the following, we will define causal rules using the idea of retrospective cohort studies.

\subsection{Causal rule definition}
Given an association rule as a hypothesis that the LHS of the rule causes its RHS. The variable of the LHS is an exposure variable and the variable of the RHS is the response variable. Let all other variables be included in the control variable set initially. We will discuss how to refine this control variable set  in Section \ref{sectionControlVariable}.

\subsubsection{Fair data sets}
Given a data set $D$, for an exposure variable, we use the following process to select samples for the exposure and non-exposure groups (while the RHS response is blinded). We firstly pick up a record $t_i$ containing the LHS factor ($P=1$), and then pick up another record $t_j$ of which $P=0$, and both $t_i$ and $t_j$ have the ``matched'' values for all the control variables. Then $t_i$ is added to the exposure group, $t_j$ is added to the non-exposure group, and both are removed from the original data set. This process repeats until no more matched pairs can be found. As a result, the distributions of the control variables in the exposure and non-exposure groups are identical or similar to each other.

We formulate the above discussions as the following definition.

\begin{definition}[Matched record pair] Given an association rule $p \to z$ and a set of control variables $C$, a pair of records  match if one contains value $p$, the other does not, and both have the matched values for $C$ according to certain similarity measure.
\end{definition}

The simplest matching is the exact matching, in which we require a pair of records have exactly the same values for control variables. For example, assume that $C = (A, B, E)$ is the control variable set for association rule $p \to z$, then records $(P=1, A=1, B=0, E=1)$ and $(P=0, A=1, B=0, E=1)$ form a matched pair. Many other similarity measures can be used for finding matched pairs of records, e.g. Euclidean distance, Jaccard distance \cite{DataMiningBookHan2005}, Mahalanobis distance  and propensity score \cite{stuart2010}, each having its own merit and disadvantages. As this paper is focused on developing and evaluating the idea of integrating association rule mining and cohort studies for causal discovery, we do not conduct extensive investigation on the different matching methods, and in our experiments, we use the exact matching and compare it with Jaccard distance matching.

\begin{definition} \label{DefFairDataSet}[Fair data set for a rule] Given an association rule $p \to z$ that has been identified from a data set $D$ and a set of control variables $C$, the fair data set $D_f$ for the rule is the maximum sub data set of $D$ that contains only matched record pairs from $D$.
\end{definition} 

\begin{example}
\label{ex_matches}
Given an association rule $a \to z$ identified using the following data set, and the control variable set $C = (M, F, H, U, P)$,  where $M$ stands for Male, $F$ for Female, $H$ for High school graduate, $U$ for Undergraduate, and $P$ for Postgraduate.

\begin{center}
\begin{tabular}{|c|cccccc|c|}
\hline
ID & A & M & F & H & U & P & Z \\
\hline
1 & 1 & 0 & 1 & 0 & 0 & 1 & 1 \\
2 & 1 & 0 & 1 & 0 & 1 & 0 & 1 \\
3 & 1 & 1 & 0 & 1 & 0 & 0 & 0 \\
4 & 1 & 1 & 0 & 0 & 0 & 1 & 1 \\
5 & 0 & 0 & 1 & 0 & 0 & 1 & 0 \\
6 & 0 & 0 & 1 & 0 & 1 & 0 & 0 \\
7 & 0 & 1 & 0 & 1 & 0 & 0 & 0 \\
8 & 0 & 1 & 0 & 1 & 0 & 0 & 1 \\
\hline
\end{tabular}
\end{center}

With exact matching, records (\#1, \#5), (\#2, \#6) and (\#3, \#7) form three matched pairs. A fair data set for $a \to z$ includes records (\#1, \#2, \#3 \#5, \#6, \#7).
\end{example}

In the above definition, the requirement of the maximum sub data set of $D$ is for the best utilisation of the data set. %However, even with this requirement, the number of records in a fair data set may not be sufficiently large in order to obtain statistically significant results.% based on a fair data set, a minimum number of matched pairs is required to be contained in the set. More details on this will be provided in Section~\ref{sec_support}.

Matches in a data set are not unique. A record may match more than one record. For example, (\#3, \#7) and (\#3, \#8) both are matched pairs (in terms of record \#3). When there are two or more possible matches, we select a matched record randomly without knowing the value of  $Z$. In the experiments, we show that such  a random selection will cause variance in the results (different causal rules validated in different runs), so we pick frequently supported rules in multiple runs to reduce the variance. However, the experiments also show that the variance is small in large data sets (one to two rule difference in three runs). Even in a small data set, more than 80\% rules are consistent over different runs.

Since with a fair data set for a rule the exposure and non-exposure groups are identical or similar except for the value of the exposure variable, if there is a significant difference in the values of the response value between the two groups, it is reasonable to assume that the difference of the outcome is caused by the difference of the values of the exposure variable.

Next, we discuss how to detect the statistical difference of the values of the response variable between the exposure and non-exposure groups, which will provide us the method for testing whether an association rule is a causal rule or not.

\subsubsection{Causal rules}\label{subsectionCausalRules}
When the values of the response variable are taken into consideration, there are four possibilities for a matched pair: both records containing $z$, neither containing $z$, record ($P=1$) containing $z$ and record ($P=0$) not; record ($P=0$) containing $z$ and record ($P=1$) not. The counts of the four different types of matched pairs in the fair data set for rule $p \to z$ can be represented as the following:

\begin{center}
\begin{tabular}{|c|c|c|}
\hline
 & \multicolumn{2}{|c|}{$P=0$} \\
\hline $P=1$ & $z$ & $\neg z$ \\
\hline $z$ & $n_{11}$ & $n_{12}$ \\
\hline $\neg z$ & $n_{21}$ & $n_{22}$ \\
\hline
\end{tabular}
\end{center}

In this table $n_{11}$ is the number of matched pairs containing $z$ in both the exposure and non-exposure groups; $n_{12}$ the number of matched pairs containing $z$ in the exposure group and $\neg z$ in the non-exposure group; $n_{21}$ the number of matched pairs containing $\neg z$ in the exposure group and $z$ in the non-exposure group; and $n_{22}$ the number of matched pairs containing $\neg z$ in both the exposure and non-exposure groups. In Example \ref{ex_matches},  $n_{11}=0$,  $n_{12}=2$, $n_{21}=0$, and $n_{22}=1$.

Using the above notation, we can have the following definition \cite{Fleiss2003}:
\begin{definition}[Odds ratio of a rule on its fair data set] \label{ORFairDataSet} The odds ratio of an association rule $p \to z$ on its fair data set $D_f$ is:
\begin{equation}\label{oratioFairData}
\od_{D_f} (p \to z) = \frac{n_{12}}{n_{21}}
\end{equation}
\end{definition}

In our experiments, we replace zero count by 1 to avoid infinite odds ratios.

The above definition leads to the definition of a causal rule:
\begin{definition}[Causal rule] \label{DefCArule} An association rule $(p \to z)$ indicates a causal relationship between $P$ and $Z$ (the variables for its LHS and RHS) and thus is called a causal rule, if its odds ratio on its fair data set,  $\od_{D_f} (p \to z) > min\_oratio$, where $min\_oratio$ is the minimum odds ratio.
\end{definition}

%We will discuss how to test if $\od_{D_f} (p \to z)$  is significantly greater than 1 in Section \ref{sectionCRdetection}.

Alternatively, to check if an association rule is a causal rule, we can use the significance test on the odds ratio of the rule on its fair data set with matched pairs. Let $\od_{D_f} (p \to z) = \omega'$ in the fair data set, the confidence interval of the odds ratio for matched pairs is defined as ~\cite{Fleiss2003}:

 \[
\omega_-=\exp(\ln{\omega}-z' \sqrt{\frac{1}{n_{12}}+\frac{1}{n_{21}}})
\]
\noindent and
\[
\omega_+=\exp(\ln{\omega}+z' \sqrt{\frac{1}{n_{12}}+\frac{1}{n_{21}}})
\]

\noindent where $z'$ is the critical value corresponding to a desired level of confidence ($z'=1.96$ for 95\% confidence) and $\omega'_-$ is the lower bound of $\od_{D_f} (p \to z)$ in the confidence level. If $\omega'_- > 1$, the odds ratio is significantly higher than 1, then we conclude that $P$ is a cause of $Z$. %This has been implemented by Line 14 in Algorithm 1.

Based on Definition \ref{DefCArule}, testing if an association rule is a causal rule becomes the problem of finding the fair data set for the rule. A fair data set simulates the controlled environment for testing the causal hypothesis represented by an association rule. When the odds ratio of an association rule on its fair data set is significantly greater than 1, it means that a change of the response variable is resulted from the change of the exposure variable. We provide further justifications in the following section.

 \subsection{Justifications for the definition of causal rules}\label{sectionPotentialOucome}

The potential outcome or counterfactual model~\cite{Pearl2000,Morgan2007} is a major framework for causal inference and it is widely used in social science, health and medical research. In this section, we will demonstrate that the causal rules defined over a fair data set is consistent with the causal relationships modelled under the potential outcome framework.

In the potential outcome model, each individual $i$ in a population has two potential outcomes with respect to a treatment: when taking the treatment ($T_i=1$), the potential outcome is $Z^1_i$; and when not taking the treatment ($T_i=0$), the potential outcome is $Z^0_i$, where $Z^1_i$ and $Z^0_i$ are random variables taking values in $\{0, 1\}$. $Z^j_i=1$ ($j\in\{0,1\}$) stands for an outcome of interest, such as a recovery.

In practice, we are only able to observe one potential outcome ($Z^1_i$ or $Z^0_i$) since an individual can only be placed in either the treatment group ($T_i=1$) or the control group ($T_i=0$), and the other potential outcome will need to be estimated.  For example, if we know that Jack did not take Panadol (i.e. $T_i=0$ considering Panadol is the treatment), and now he gets a high temperature (i.e. $Z^0_i=1$ assuming high temperature is an outcome), the question that we are asking is, what the outcome would be if Jack had taken Panadol, i.e. we want to know the potential outcome $Z^1_i$. So the potential outcome model is also called counterfactual model.

Let us assume that we have both $Z^1_i$ and $Z^0_i$ of an individual $i$. With the potential outcome model, the causal effect of the treatment on $i$ is defined as:
\begin{equation}\label{causalEffect}
\delta_i = Z^1_i - Z^0_i
\end{equation}

We often aggregate the causal effects on individuals in the population (or samples) and obtain the average causal effect as the following, where $E[.]$ is the expectation of a random variable.
\begin{equation}\label{aggCausalEffect}
E[\delta_i] = E[Z^1_i] - E[Z^0_i]
\end{equation}
In the above equation $i$ is kept as in other work on the counterfactual framework, to indicate individual level heterogeneity of potential outcomes and causal effects~\cite{Morgan2007}.

To link the above discussion to our definition of causal rules, treatment $T$ and $Z^j_i$ ($j\in\{0,1\}$) are the exposure variable $P$ and the response variable $Z$ respectively in the causal rule definition. In the following, we keep using the notation of the potential outcome framework.

%The average causal effect is different from support different $E(Z | T=1, D) - E(Z | T=0, D)$ in data set $D$. $Z \in \{1, 0\}$ is the outcome variable, and we do not use superscripts since $Z$ are observed outcomes not potential outcomes in data set $D$. Support difference indicates count difference between the treatment and control groups, but it is not consistent with causal effect as shown in examples in the introduction. Support difference is consistent to the causal effect only when data is obtained from a well designed randomised trial.

Since we are only able to observe one of the two potential outcomes for each individual $i$, the causal effect in equation (\ref{aggCausalEffect}) cannot be estimated from any data set directly. However, it can be estimated under a perfect stratification of the data~\cite{Morgan2007}, where for a stratum samples within treatment and control groups are collectively indistinguishable from each other on the values of the stratifying variables and the samples are only different on the observed treatment status. Furthermore the outcome status of a sample is purely random. In this case, we can assume that:
\begin{equation}\label{eqnCausalEffect1}
E[Z^1_i\mid T_i=0, D_{ps}] = E[Z^1_i\mid T_i=1, D_{ps}]
\end{equation}
\begin{equation}\label{eqnCausalEffect2}
E[Z^0_i\mid T_i=1, D_{ps}] = E[Z^0_i\mid T_i=0, D_{ps}]
\end{equation}
\noindent where $S$ represents that the data set is perfectly stratified using the stratifying variables.

The above equations indicate that the potential outcome of an individual taking a treatment (in fact she/he has not) can be estimated by the `real' outcome of the matched individual who has taken the treatment. Similarly, the potential outcome of an individual not taking a treatment (in fact she/he has taken) can be estimated by the `real' outcome of the matched individual who has not taken the treatment.

%Therefore with a data set under a perfect stratification, equation \ref{{eqnCausalEffect} becomes
%E[\delta_i] = E[Z^1_i] - E[Z^0_i]

Samples in a fair data set in fact are perfectly stratified, as samples in the exposure and non-exposure groups have the same distribution in terms of the values of control variables, and the value of the response variable of a sample in the exposure or in the non-exposure group is random. Therefore according to equations (\ref{eqnCausalEffect1}) and (\ref{eqnCausalEffect2}), for a fair data set $D_f$, we have:
\begin{equation}\label{eqnCausalEffectF1}
E[Z^1_i \mid T_i=0, D_f] = E[Z^1_i \mid T_i=1, D_f]
\end{equation}
\begin{equation}\label{eqnCausalEffectF2}
E[Z^0_i \mid T_i=1, D_f] = E[Z^0_i \mid T_i=0, D_f]
\end{equation}

Let us now show how to estimate the causal effect, $E[\delta_i]$ with a fair data set. In a fair data set, the number of individuals being treated is the same as the number of individuals not being treated. Therefore the average causal effect can be represented as the following:
\begin{equation}\label{averageCausalEffectFair}
E[\delta_i]_{D_f} \!=\! \frac{1}{2} (E[Z^1_i \mid T_i\!=\!1,\! D_f] - E[Z^0_i \mid T_i\!=\!1, \!D_f]) + \frac{1}{2} (E[Z^1_i \mid T_i\!=\!0, \!D_f] - E[Z^0_i \mid T_i\!=\!0, \!D_f])
\end{equation}

In the above formula, based on equations (\ref{eqnCausalEffectF1}) and (\ref{eqnCausalEffectF2}), we substitute $E[Z^0_i \mid T_i=0, D_f]$ and $E[Z^1_i \mid T_i=1, D_f]$ for $E[Z^0_i \mid T_i=1, D_f]$ and $E[Z^1_i \mid T_i=0, D_f]$ respectively. As a result, the average causal effect in the fair data set is estimated as the following:
\begin{equation}\label{averageCausalEffectFair1}
E[\delta_i]_{D_f} = E[Z^1_i \mid T_i=1, D_f] - E[Z^0_i \mid T_i=0, D_f]
\end{equation}

\noindent where both outcomes are observable. So when there is no sample bias, we can remove the superscripts and subscripts and obtain the average causal effect of the samples (or a population) as the following:
\begin{equation}\label{averageCausalEffectFair2}
\Delta = E[Z \mid T=1, D_f] - E[Z \mid T=0, D_f]
\end{equation}

%In fact, matching is a fundamental method for estimating cause effects in the potential outcome framework~[Morgan~Matching Estimators of Causal Effects--Prospects and Pitfalls in Theory and Practice].

This formula suggests that following the potential outcome model, the causal effect is the difference of the outcomes in the treatment (exposure) group and the control (non-exposure) group in a fair data set. %The difference in the fair data set is a consistent estimation of the causal effect.
In our definition of a causal rule, we also use the difference of outcomes in different groups to identify causal rules, except that we use the odds ratio to represent the difference as a cohort study does instead of the above arithmetic difference. %When we replace $T$ with our $P$ and quantify the difference of outcome $Z$ in the fair data set, the left hand sides of causal rules are the variables that result in cause effect in the outcome variable.
Therefore, the definition of a causal rule over a fair data set is correct, in the sense that it is consistent with the approach under the potential outcome framework.

\section{Algorithm}\label{sectionAlg}
In this section we present the algorithm (Algorithm~1) for causal rule mining (called CR-CS in the rest of this paper). The algorithm integrates association rule mining with causal relationship test based on cohort studies. In the following, we firstly discuss two anti-monotone properties for efficient generation of candidate causal rules,  and we then discuss the selection of control variables for building a fair data set. Finally, we introduce the details of detecting causal rules from the candidate causal rules.
\begin{algorithm}[h]
\label{alg-Naive}
\caption{Causal Rule mining with Cohort Study (CR-CS)}
Input: Data set $D$ with the response variable $Z$, the minimal local support $\delta$, the maximum length of rules $k_0$, and the minimum odds ratio $\alpha$. \\
Output: A set of causal rules
\begin{algorithmic}[1]
\STATE let causal rule set $R_C = \emptyset$
\STATE add 1-patterns to a prefix tree $T$ (see Section \ref{sectionCandidateRule}) as the 1st level nodes
\STATE count support of the 1st level nodes with and without response $z$
\STATE remove nodes whose local support is no more than $\delta$ $\slash\slash$ Support pruning
\STATE Let $X$ be the set of attributes containing frequent 1-patterns
\STATE find the set of irrelevant attributes $I$
\STATE let $k=1$
\WHILE {$k \le k_0$}
\STATE generate association rules at the $k$-th level of $T$
\FOR {each generated rule $r_i$}
\STATE find exclusive variables $E$ of $LHS(r_i)$
\STATE let control variable set $C = X \backslash (I, E, LHS(r_i))$
\STATE create a fair data set for $r_i$ $\slash\slash$ Function 1
\IF {$\od_{D_f}(r_i) > \alpha$}
\STATE move $r_i$ to $R_C$
\STATE remove $LHS(r_i)$ from the $k$-th level of $T$ $\slash\slash$ Observation 1
\ENDIF
\ENDFOR
\STATE $k=k+1$
\STATE generate $k$-th level nodes of $T$
\STATE count the support of the $k$-th level nodes with and without response $z$
\STATE remove nodes whose local support is no more than $\delta$ $\slash\slash$ Support pruning
\STATE remove nodes of patterns whose supports are the same as those of their sub-patterns respectively $\slash\slash$ Observation 2
\ENDWHILE
\STATE {output $R_C$}
\end{algorithmic}
\end{algorithm}
\subsection{Anti-monotone properties}
\label{sec_anti-monotone}
Anti-monotone properties are at the core for efficient association rule mining. For example a well known anti-monotone property is that a super set of an infrequent pattern is infrequent, and infrequent patterns are pruned before they are generated (called forward pruning). We firstly discuss the anti-monotone properties that we will apply to candidate causal rule pruning.

In the following discussions, we say that rule $px \to z$ is more specific than rule $p \to z$, or $p \to z$ is more general than $px \to z$. Furthermore, we use $\cov(p)$ to represent the set of records in $D$ containing value $p$, and we call $\cov(p)$ the covering set of $p$. A rule is \emph{redundant} if it is implied by one of its more general rules. %For example, if a causal relationship of a rule is indicated by a more general rule, then the former rule is redundant. A $k$-pattern is a pattern containing $k$ values.

\begin{observation}[Anti-monotone property 1]
\label{observation_1}
All more specific rules of a causal rule are redundant.
\end{observation}

\begin{proof}
This observation is based on the persistence property of a real causal relationship. Persistence means that a causal relationship holds in any condition. This implies that when a rule is specified, although additional conditions are added to the LHS of the rule, the conditions do not change the causal relationship. Therefore for the purpose of discovering causal rules/relationships, more specific candidate causal rules are implied by the general rule, and hence are redundant.
\end{proof}

For example, if rule ``college graduate $\to$ high salary'' holds, then we know that both male college graduates and female college graduates enjoy high salaries. It is therefore redundant to have the rules ``male college graduate $\to$ high salary'' and ``female college graduate $\to$ high salary''.

\begin{observation}[Anti-monotone property 2]
\label{observation_2}
If $\supp(px)=\supp(p)$, rule $px \to z$ and all more specific rules of $px \to z$ are redundant.
\end{observation}

\begin{proof}
If $\supp(px)=\supp(p)$, then $\cov(px) = cov(p)$. In other words, both $p \to z$ and $px \to z$ cover the same set of records. There will be the same fair data set for both rules. Therefore, if $p \to z$ is a causal rule, so is $px \to z$. If $p \to z$ is not a causal rule, nor is $px \to z$.  Hence rule $px \to z$ is redundant.

Let rule $pxy \to z$ be a more specific rule of rule $px \to z$. If $\supp(px) = \supp(p)$, then $\supp(pxy) = \supp(py)$. Using the same reasoning above, we conclude that rule $pxy \to z$ is redundant with respect to rule $px \to z$.
\end{proof}

Since there are two anti-monotone properties in addition to the anti-monotone property of support, it is efficient to use a level wise algorithm like Apriori \cite{Apriori}. Both anti-monotone properties \ref{observation_1} and \ref{observation_2} can be used in the same way as the anti-monotone property of support.

\subsection{Control variables}\label{sectionControlVariable}

The set of control variables determines the size of a fair data set. If the control variable set is large, the chance of finding a non-empty fair data set is small. Therefore we need to find a proper control variable set, without compromising the quality of the causal discovery. In the following we discuss how to obtain such a control variable set.

Let $X$ represent the set of all predictor variables, and as before $P$ is the exposure variable and $C$ is a set of control variables. Initially, let $C=X \backslash P$.

\begin{definition}[Relevant and irrelevant variables]
If a variable is associated with the response variable, it is relevant. Otherwise, it is irrelevant.
\end{definition}

We do not control irrelevant variables, hence $C=X \backslash (P, I)$ where $I$ stands for a set of irrelevant variables.

The major purpose for controlling is to eliminate the effects of other possible causal factors on the response variable. Other variables that are random with respect to the value of the response variable can be considered as noises and need not to be controlled. With Example~\ref{ex_simpsionParadox}, when we test the association rule ``Gender = $m$'' $\to$ ``Salary = $low$'' for finding a causal relationship, we should control variables like education, location, profession and working experience. However, we do not control variables like blood type and eye colour, since they are irrelevant to salary.

The combination of multiple irrelevant variables can be relevant. However, we do not consider combined variables in the control variable set. There will be many combined relevant variables and the support of combined variables are normally small. Therefore when they are included in the control variable set, it is very likely to have empty exposure or non-exposure groups.

\begin{definition}[Exclusive variables]
\label{def_exclusive}
Variables $P$ and $Q$ are mutually exclusive if $\supp(pq) \le \epsilon$ or $\supp(\neg p q) \le \epsilon$ where $\epsilon$ is a small integer.
\end{definition}

We do not control an exclusive variable of the exposure variable $P$, i.e. we let $C=X \backslash (P, I, Q)$ where $Q$ stands for a set of exclusive variables of $P$. Because if an exclusive variable is controlled, the exposure group or the non-exposure group may be empty, thus we are unable to do a cohort study. Let us take $\epsilon = 0$ as an example. When $\supp(pq) = 0$, we will have samples with $(P=1, Q=0)$, $(P=0, Q=1)$ and $(P=0, Q=0)$, but not $(P=1, Q=1)$. In this case, for a record in the non-exposure group with $(P=0, Q=1)$, no match can be found in the exposure group with $(P=1, Q=1)$. When $\supp(\neg p q) = 0$, we will have samples with $(P=1, Q=1)$, $(P=1, Q=0)$ and $(P=0, Q=0)$, but not $(P=0, Q=1)$, then for a record in the exposure group with $(P=1, Q=1)$, it is impossible to find a match with $(P=0, Q=1)$.

A main type of exclusive variables are those caused by database constraints. For example, $P$ represents the highest qualification being high school and $Q$ represents the highest qualification being university degree. As they both belong to the same domain in a relational data set, and an individual has only one highest qualification, $P$ and $Q$ are mutually exclusive ($\supp(pq) = 0$). In this case, it is not necessary to control $Q$ as it does not affect the finding about whether $P$ is a cause of the response variable.

Another type of exclusive variables are redundant attributes. Let us assume that two variables have the identical values but different names. e.g. $P$ and $Q$. They are mutually exclusive since $\supp(\neg p q) \le \epsilon$. We do not need to test both separately to see if they are causes of the response variable since one test is enough. However, if we include $Q$ in the control variable set, we will not be able to test $P$ since the fair data set is empty.

Exclusive variables can be confounding variables, for example $P=thunder$ and $Q=storm$ may be mutually exclusive in a data set since $\supp(\neg p q) \le \epsilon$. Let us assume that they jointly cause the response. If we control $Q$,  $P$ will not be tested as a cause. When we remove $Q$ from the control variable set, we will be able to find $P$ as a cause. It is not difficult to find out that $Q$ is a confounder of $P$ in post processing since they are strongly associated.

\subsection{Candidate causal rule generation}\label{sectionCandidateRule}

This algorithm makes use of branch and bound search similar to Apriori \cite{Apriori} for association rule mining. The algorithm employs support pruning plus the two pruning criteria (Observations 1 and 2) presented in Section~\ref{sec_anti-monotone}, and therefore searches much smaller search space than Apriori. The algorithm is based on a prefix tree structure for candidate generation, storage and counting. The prefix tree structure has been shown to support efficient implementation for branch and bound search \cite{Borgelt03Apriori}.

A prefix tree is an ordered tree to store ordered sets (see Figure \ref{prefixtree} for an example). In our algorithm, each node stores a set of nonzero variable values (or a potential LHS of a rule). We assume that nonzero variable values are coded and ordered, and this is to prevent generating duplicate candidate causal rules. A node stores the prefix set of the sets stored in its child nodes, and a child node is labeled by the different value between its stored set and the set of its parent. The root of the prefix tree stores an empty set.
\begin{figure}
\center
\includegraphics[width=0.70\textwidth]{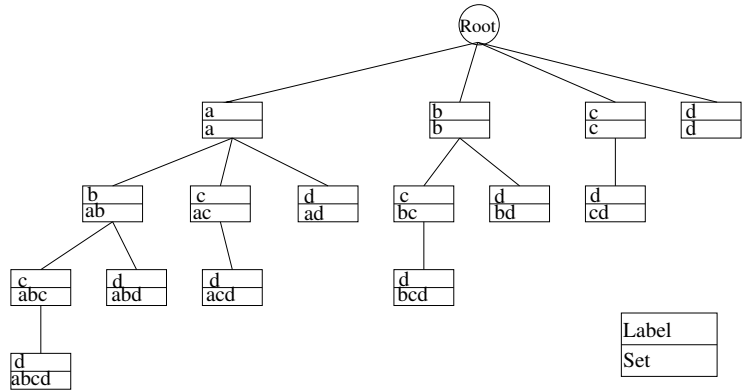}
\caption{An example prefix tree. $a, b, c, d$ stand for nonzero values of variables $A, B, C$ and $D$.}
\label{prefixtree}
\end{figure}

The set of labels along the path from the root to a node is the set of values stored in the node. This index makes the search, insert, and delete a node easy. This property also makes the counting of supports of value sets efficient too. If a set stored in a parent node is not in a record, all sets stored in its child nodes will not be in the record too. This property avoids many unnecessary search in the counting process. In the counting process, each node store two counts, for example, $\supp(pz)$ and $\supp(p \neg z)$. As a result, the contingency table of $p \to z$ is determined.

All super sets with the same prefix stored under a parent node storing the prefix. When the parent node is removed, so are all super sets. This suits the forward pruning for candidate causal rule generation very well. For more efficient pruning, the backtrack links to other parent nodes in the prefix tree are also added. For example, in Figure \ref{prefixtree} node $abc$ links to node $bc$, and such link will facilitate the pruning by Anti-monotone property 2.

Referring to Algorithm 1, the code involved in the candidate generation includes Lines 2-4, 9, and 20-23. They are self-explanatory based on the discussions in this and the previous subsection.

%{\color{red} The text discussing lower bound has been removed.}
\iffalse
{ 
We only explain Line 9 here. Since the response, $z$ is fixed, one node contains only one candidate rule. The test of significant association is determined as the following.

Let $\omega$ be the odds ratio of the rule $p \to z$ on the given data set $D$, i.e. $\od_{\textsc{d}} (p \to z) = \omega$. The confidence interval of $\omega$, $[\omega_{-}, \omega_{+}]$, is defined as \cite{Fleiss2003}:\\

\[
\omega_-=\exp(\ln{\omega}-z' \sqrt{\frac{1}{n_{11}}+\frac{1}{n_{12}}+\frac{1}{n_{21}}+\frac{1}{n_{22}}})
\]
\noindent and
\[
\omega_+=\exp(\ln{\omega}+z' \sqrt{\frac{1}{n_{11}}+\frac{1}{n_{12}}+\frac{1}{n_{21}}+\frac{1}{n_{22}}})
\]

\noindent where $z'$ is the critical value corresponding to a desired level of confidence ($z'=1.96$ for 95\% confidence). $\omega_-$ and $\omega_+$ are  the lower and upper bounds respectively of an odds ratio at a confidence level. If $\omega_- > 1$, the odds ratio is significantly higher than 1, hence $P$ and $Z$ are associated. Equivalently, $p \to z$ is an association rule. Therefore, we do not use the minimum odds ratio in the algorithm.

Another advantage of the above process is that it is automatically adaptive to the size of a data set. For a large data set, the confidence interval of an odds ratio is small and hence a small odds ratio can be significantly higher than 1. For a small data set, the confidence interval of an odds ratio is large and hence a large odds ratio is needed to be significantly higher than 1.
}
\fi
\subsection{Causal rule detection}\label{sectionCRdetection}
%\label{sec_testingcausal}
This process involves three steps, as discussed below.

\subsubsection{Determining control variables}
We firstly determine the set of irrelevant variables, each of which is not associated with the response variable. For a variable $Y$, its association with the response variable $Z$ can be determined by the odds ratio of $y \to z$.
%The significance of association is tested in the same way as for $p \to z$ discussed above.
The above identified irrelevant variables are excluded from the control variable set. These are implemented by Lines 6 and 12.

%{\color{red} This paragraph needs revision after the new implementation.}
Secondly we identify the exclusive variables of an exposure variable, say $P$, according to Definition~\ref{def_exclusive} where $\epsilon$ is set to the same value as the minimum local support.  We exclude the identified exclusive variables from the control variable set.  %Note that we do not consider negative association in this work since we have itemized the attribute values and we are only interested in the values of 1 not of 0. Associations are determined in the same way as discussed above.
These are implemented by Lines 11 and 12.

The remaining variables then form the control variable set. The control variable set can be viewed as a set of patterns in association rule mining. For example, if male, female, college and postgraduate form the control variable set, the set includes the patterns \{(male, college), (male, postgraduate), (female, college), (female, postgraduate)\}.

\subsubsection{Creating fair data set}
\label{sec_support}
We select the samples from the given data set $D$ to get the fair data set for rule $p \to z$, following the procedure listed in Function~1. We firstly find the covering set of $c$. Then the covering set of $c$ is split into two subsets: one containing value $p$, denoted by $D_{cp}$, and the other containing value $\neg p$ (or $P=0$), denoted by $D_{c\neg p}$. Assume that $|D_{cp}| \le |D_{c\neg p}|$ (if not, we swap the order in the following description). For each record in $D_{cp}$, find a matched record in $D_{c\neg p}$ with respect to the control variable set. We have implemented exact matching and the matching using Jaccard distance. If there are more than one matched records, choose one randomly. Add the pair of records to the fair data set. If there is no matched record in $D_{c\neg p}$, move to the next record. %This process is implemented by Function 1.

\begin{algorithm}[tb]
\textbf{
%{\color{red} This function needs revision after the new implementation.}
Function 1 Create a fair data set for rule $p \to z$}\\
Input: Data set $D$, rule $p \to z$, and control variable set $C$\\
Output: a fair data set for rule $p \to z$, $D_f$
\begin{algorithmic}[1]
\STATE find the covering set of $c (C=1)$, $D_c$
\STATE split $D_c$ into $D_{cp}$ and $D_{c\neg p}$ $\slash\slash$ $D_{cp}$ contains value $p$ and $D_{c\neg p}$ does not
\STATE let $D_f = \emptyset$
\FOR {each record $t_i$ in $D_{cp}$  $\slash\slash$ assuming $|D_{cp}| \le  |D_{c\neg p}|$. If not, swap $D_{cp}$ and $D_{c\neg p}$.}
\FOR {each record $t_j$ in $D_{c \neg p}$}
\IF {$t_i$ and $t_j$ are matched w.r.t the values of $C$}
\STATE move $t_i$ and $t_j$ to $D_f$
\ENDIF
\ENDFOR
\ENDFOR
\STATE {output $D_f$}
\end{algorithmic}
\label{alg_fairDataset}
\end{algorithm}

%{\color{red} The section of Testing causal rules has been removed.}

\subsubsection{Testing causal rules}
To check if an association rule $p \to z$ is a causal rule, we firstly follow Definition \ref{ORFairDataSet} to calculate the odds ratio of the rule on its fair data set created in the previous step. Then according to Definition \ref{DefCArule}, if the odds ratio is greater than the given minimum odds ratio, we can say that $p \to z$ is a causal rule. This has been implemented by Line 14 in Algorithm 1. Alternatively, we can use the method introduced in Section \ref{subsectionCausalRules} to test the significance of the odds ratio of the rule on its the fair data set. If the odds ratio is significantly higher than 1 for a given confidence level, then we conclude that $P$ is a cause of $Z$.

\section{Experiments}

\subsection{Data sets and parameters}\label{subsec_Datasets}

To evaluate CR-CS, the proposed causal rule mining algorithm, twenty four synthetic data sets and {  eight} frequently used public data sets were employed in the experiments. A summary of the data sets is given in Table \ref{tab_description}. The number of variables in the table refers to the number of predictor variables in a data set. All predictor variables and the response variable are binary variables, with values of 1 or 0 indicating the presence or absence of an attribute correspondingly. The class variable in each of the eight public data sets is set as the response variable in our experiments. The distributions refer to the percentages of the two different values of response variables in the data sets. For the synthetic data sets, the ground truth column represents the number of true single causes and known combined causes each consisting of two predictor variables.
\begin{table}
\begin{center}
\caption{A summary of data sets used in experiments}
\label{tab_description}
\begin{tabular}{|c|c|c|c|c|}
\hline Name & \#Records & \#Variables & Distributions &  Ground Truth \\
\hline
%{ Asia} & { 4000} & { 7} & { 5.9\% \& 94.1\%} & { 4}\\
%{ Cancer} & { 4000} & { 4} & { 1.8\% \& 98.2\%} & { 4}\\
%{ Earthquake} & { 4000} & { 4} & { 1.3\% \& 98.7\%} & { 4}\\
%{ Metastasis} & { 4000} & { 4} & { 32.5\% \& 67.5\%} & { 3}\\
{V20-2K } & {2000} & {19} & {41.9\% \& 58.1\%} & {7} \\
{V20-5K } & { 5000} & { 19} & { 41.9\% \& 58.1\%} & { 7} \\
{ V20-10K } & { 10000} & { 19} & { 41.9\% \& 58.1\%} & { 7} \\
{ V40-2K } & { 2000} & { 39} & { 37.6\% \& 62.4\%} & { 7} \\
{ V40-5K } & { 5000} & { 39} & { 37.6\% \& 62.4\%} & { 7} \\
{ V40-10K } & { 10000} & { 39} & { 37.6\% \& 62.4\%} & { 7} \\
{ V60-2K } & { 2000} & { 59} & { 52.5\% \& 47.5\%} & { 7} \\
{ V60-5K } & { 5000} & { 59} & { 52.5\% \& 47.5\%} & { 7} \\
{ V60-10K } & { 10000} & { 59} & { 52.5\% \& 47.5\%} & { 7} \\
{ V80-2K } & { 2000} & { 79} & { 50.6\% \& 49.4\%} & { 8} \\
{ V80-5K } & { 5000} & { 79} & { 50.6\% \& 49.4\%} & { 8} \\
{ V80-10K } & { 10000} & { 79} & { 50.6\% \& 49.4\%} & { 8} \\
{ V100-2K } & { 2000} & { 99} & { 48.1\% \& 51.9\%} & { 6} \\
{ V100-5K } & { 5000} & { 99} & { 48.1\% \& 51.9\%} & { 6} \\
{ V100-10K } & { 10000} & { 99} & { 48.1\% \& 51.9\%} & { 6} \\
%\hline
%{ Syn100-5K-Dense} & { 5000} & { 99} & { 19.8\% \& 80.2\%} & { 20}\\
%{ Syn120-5K-Dense} & { 5000} & { 119} & { 19.8\% \& 80.2\%} & { 26}\\
\hline
{ V200-10K} & { 10000} & { 199} & { 19.8\% \& 80.2\%} & { 20} \\
{ V400-10K} & { 10000} & { 399} & { 19.8\% \& 80.2\%} & { 40}\\
{ V600-10K} & { 10000} & { 599} & { 19.8\% \& 80.2\%} & { 60}\\
{ V800-10K} & { 10000} & { 799} & { 19.8\% \& 80.2\%} & { 80}\\
{ V1000-10K} & { 10000} & { 999} & { 19.8\% \& 80.2\%} & { 100}\\
\hline \hline  Name &  \#Records &  \#Variables &  Distributions &   Known combined rules \\
\hline
{ V8-2K} & { 2000} & { 7} & { 45.1\% \& 54.9\%} & { 2}\\
{ V12-2K} & { 2000} & { 11} & { 72.1\% \& 27.9\%} & { 3}\\
{ V16-2K} & { 2000} & { 15} & { 45.2\% \& 54.8\%} & { 3}\\
{ V20-2K-cmb} & { 2000} & { 19} & { 55.6\% \& 44.4\%} & { 4}\\
\hline \hline
%{  BCW} & { 700} & { 90} & { 34.5\% \& 65.5\%} & { -}\\
{ German} & { 1000} & { 60} & { 30.0\% \& 70.0\%} & { -}\\
{ Kr-vs-kp} & { 3196} & { 74} & { 47.8\% \& 52.2\%} & { -}\\
{ Mushroom} & { 8124} & { 215} & { 48.2\% \& 51.8\%} & { -}\\
{ Tic-tac} & { 958} & { 27} & { 34.7\% \& 65.3\%} & { -}\\
%{ Vote} & { 435} & { 48} & { 38.6\% \& 61.4\%} & { -}\\
Adult & 48842 & 99 & 23.9\% \& 76.1\% & { -}\\
Hypothyroid & 3163 & 51 & 4.8\% \& 95.2\% & { -}\\
Sick & 2800 & 58 & 6.1\% \& 93.9\% & { -}\\
Census income & 299285 & 495 & 6.2\% \& 93.8\% & { -}\\
\hline
\end{tabular}
\end{center}
\end{table}

{The first fifteen synthetic data sets in Table \ref{tab_description} were used to evaluate the performance of CR-CS in finding single causal rules in comparison with the Bayesian network based methods, PC-Select, CCC, and CCU. Those synthetic data sets  of random Bayesian networks  were generated using the TETRAD software (http://www.phil.cmu.edu/tetrad/). In TETRAD, we firstly generate randomly the structure of the BN  using the ``simulate data from IM" template. The conditional probability table was also randomly assigned, which will be used to simulate the data. The data sets were then generated  using the built-in Bayes Instantiated Model (Bayes IM). In the Bayes IM, the data of each binary variable was randomly generated so that the distributions of all the variables satisfy the  constraints in the conditional probability tables. We selected a node in each of the BNs as the fixed target for running the algorithms.

The next five synthetic data sets (V200-10K, ..., V1000-10K) were used to assess the efficiency of the algorithms. To generate those large data sets with a fixed propotion of nodes being the parents of the target node (which is not practical with TETRAD), we firstly draw simple BNs where some predictor variables are parents of the response variable, and some are not. We then use logistic regression to simulate the data sets for those BNs. The total number of causes in each BN is given in Table \ref{tab_description}.

Meanwhile the four data sets, V8-2K, V12-2K, V16-2K, and V20-2K-cmb, are for assessing the ability of CR-CS to discover combined causes. These four synthetic data sets have been generated with the following procedure. We firstly generate a data set for a random BN using TETRAD and choose a node as the target. To create a known combined cause, we randomly select a parent variable, $X$, of the target in the generated BN to split it into two new variables, $X_a$, $X_b$. The new variables must satisfy two conditions: (1)$X=X_{a}\wedge X_{b}$ (i.e. $X=1$ if and only if $X_a=1$ and $X_b=1$), and (2) $X_a$ and $X_b$ are not associated with the response variable. The number of known combined causes are shown in Table \ref{tab_description}. Note that we do not have a complete ground truth of all combined causes, as there may be other combined causes in the data set due to the combinations of non-causal single variables. In the experiments, we investigate the performance of CR-CS in terms of the ability to recover known combined causes.}

Among real world data sets, Hypothyroid and Sick are two medical data sets and they were originally retrieved from the Thyroid Disease folder of the UCI Machine Learning Repository  \cite{Bache+Lichman:2013} and then discretised by using the MLC++ discretisation utility \cite{kohavi.ea:using-mlc:96}. The Adult data set is an extraction of the USA census database in 1994 and it was also retrieved from the same repository. In our experiments, all continuous attributes have been removed from the original Adult data set. These three data sets were used in the experiments for testing the effectiveness of CR-CS, in comparison with other methods (see Sections \ref{sectionCRvsAR} and \ref{sectionCRvsOther}). They were also used for evaluating the stability (Section \ref{sectionStability}) of CR-CS and the impact of different matching methods (Section \ref{sectionMatching}).

The Census Income (KDD) data set was also sourced from the UCI Machine Learning Repository. We combined the training and test data sets and then sampled 50K, 100K, 150K, 200K and 250K records for the experiments. Continuous attributes have been removed. The data set and the last five synthetic data sets (with 10K records) were used to assess the efficiency of CR-CS (Section \ref{sectionEfficiency}). {  Other real world data sets are also from UCI Machine Learning Repository and are used to investigate the number of combined causes discovered by CR-CS.}

In the experiments, while the default minimum local support ($\delta$ in Algorithm 1) was 0.05, we set it to 0.01 for the Adult data set in the comparison with the other three methods,  CCC \cite{Cooper:LCD1997}, CCU \cite{Silverstein:CCU2000} and PC-select \cite{pcalg}. {The confidence level was set to 99\% for calculating the confidence interval (lower bounds and upper bounds) of the odds ratio for synthetic data sets, and 95\% for real world data sets considering the noises in the real world data sets.}

\subsection{Causal rules vs. association rules}\label{sectionCRvsAR}
{ Causal rules have advantages over association rules. An association rule may represent a spurious relationship between two variables as a statistical association does not necessarily mean that the two variables are related or directly related (while a causal rule indicates that the two variables have a direct relationship given the observed variables). Those spurious association rules could not be removed by increasing thresholds. They can only be identified by analysing the relationship by shielding the effects of other variables.}

To investigate the difference between association rule mining and causal rule mining, we compared the results obtained by CR-CS with the results of various types of association rule mining. From Table \ref{tab_numberofrules}, the number of causal rules is significantly smaller than the numbers of other types of (association) rules, including association rules \cite{Apriori}, non-redundant rules \cite{Zaki_non_redundant04}, and optimal rules \cite{Li_On_optimal}. Associations are measured by the odds ratio defined in Definition \ref{def_oddsratio}, and their significance is tested using the method discussed in Section \ref{subsectionAR}, and all the methods used the same minimum local support. The maximum length of rules is 4.

\begin{table}
\begin{center}
\caption{Comparison of the numbers of association rules (AR), non-redundant rules (NRR), optimal rules (OR) and causal rules (CR). {  Many association and interesting rules are not causal.}}
\label{tab_numberofrules}
\begin{tabular}{|l|cccc|}
\hline
 & \#AR & \#NRR & \#OR & \#CR \\
\hline
Adult &  3108 & 2863 &  976 &  46  \\
Hypothyroid & 39476 & 17692 & 3237 &  30 \\
Sick & 56183 & 28698 & 3917 &  21 \\

\hline
\end{tabular}
\end{center}
\end{table}

The number of causal rules obtained from a data set is very small. They may not be enough for classification since not every record in the data is covered by a causal  rule. However, they are more reliable relationships since each causal rule is tested by the cohort study in data.

Most discovered causal rules (99\%) are short and include one or two variables, which makes it easy for these rules to be easily interpreted and applied to solve real world problems where only short rules are preferred.

\subsection{Causal rules vs. findings of other causal discovery methods}\label{sectionCRvsOther}

To evaluate the performance of CR-CS, we conducted a set of experiments with the first { 15} synthetic data sets and { 3 real world data sets, Adults, Hypothyroid and Sick,} and compared the performance of CR-CS with the constraint based methods, CCC \cite{Cooper:LCD1997}, CCU \cite{Silverstein:CCU2000}, and PC-select \cite{pcalg}.

As mentioned in Section \ref{sectionRelatedWork}, CCC \cite{Cooper:LCD1997} and CCU \cite{Silverstein:CCU2000} are two efficient constraint based causal discovery methods. Both of them learn the simple structures involving three variables with certain dependence/independence relationships among them, and infer causal relationships from the structures. Both methods assume no hidden and no confounding variables in data sets. PC-select \cite{pcalg} is a local causal discovery method that finds all the parents and children of a given node.  It is similar to the well-known PC algorithm \cite{Spirtes2001book} for learning a Bayesian network, except that it only finds the local causal relationships around a given response variable. The PC algorithm can return optimal result

In the experiments, CCC and CCU were restricted to identify the structures involving the response variables only. When a statistical significance test was involved, 95\% confidence level is used. With our method (CR-CS), since there are small variations in the causal rules discovered in different runs due to random selection of matched pairs when a record has multiple matches, in the experiments, with one data set, we generated causal rules (i.e. ran the algorithm) three times and chose the rules occurring at least twice in the three runs.

\subsubsection{Experiment results of synthetic data}
{  Table \ref{tab_numberCAR-CCC-CCU_syn} shows the precision ($P$), recall ($R$), and $F_1$-measure ($F_1$) of the four methods for the 15 synthetic data sets with different number of variables and samples. As we can see from the table, PC-select and CR-CS are significantly better than CCC and CCU in precision, recall, and $F_1$ measure. CR-CS and PC-select achieve good results with more than 70\% in precision and $F_1$ measure for most of the synthetic data sets.

To investigate if a method performs better than the other, for each pair of methods,  we conduct the Wilcoxon test \cite{demvsar2006statistical} of the $F_1$-measures of the results obtained by the pair of methods with the fifteen data sets. Table \ref{tab_Wil} shows the pairwise test results for the four methods. Overall, PC-select and CR-CS are significantly better than CCC and CCU, but there is no evidence to conclude that CR-CS or PC-select is better than the other. However, note that PC-select is only suitable for  data sets with small number of nodes or sparse data sets with small  number of causes of the target.  It took more than two hours for PC-select to complete when it was applied to the synthetic data set with 100 nodes and 20 causes of the target, and it failed to return results for the  data set with 120 nodes with 26 causes of the target within 24 hours.}% CCC always found all the ground truth of single causes (100\% recall for all 8 data sets), with some false positives, while the performance of CCU (considering both precision and recall) was very poor all the eight data sets. PC-select performed pleasingly well (100\% recall and precision) with data sets with 100 or less variables, but it failed to return results for the two data sets with 120 nodes within 24 hours. In terms of finding single causes, CR-CS also achieved 100\% recall and precision when the sample sizes were large enough. When a data set had smaller number of samples, the performance of CR-CS degraded because in this case it could not have sufficient number of matched pairs to generate an adequately large fair data set for getting correct results. As about discovering combined causes, when there were enough samples, CR-CS performed very well too.}

\begin{table}[t]
\begin{center}
\small
  \caption{{Performance of CCC, CCU, PC-select, and CR-CS in finding single rules with synthetic data sets. $P$, $R$ and $F_1$ represent precision, recall and $F_1$-measure, respectively.}}
\label{tab_numberCAR-CCC-CCU_syn}
\scalebox{0.95}{
    \begin{tabular}{|c|ccc|ccc|ccc|ccc|}
    \hline
          & \multicolumn{3}{c|}{ CCC} & \multicolumn{3}{c|}{ CCU} & \multicolumn{3}{c|}{ PC-select} & \multicolumn{3}{c|}{ CR-CS} \\
          &  $P$     &  $R$     &  $F_1$    &  $P$     &  $R$     &  $F_1$    &  $P$     &  $R$     &  $F_1$    &  $P$     &  $R$     &  $F_1$ \\
    \hline
     V20-2K    &  0.75  &  0.86  &  0.80  &  1.00  &  0.57  &  0.73  &  0.83  &  0.71  &  0.77  &  1.00  &  0.57  &  0.73 \\
     V20-5K   &  0.63  &  1.00  &  0.78  &  0.50  &  0.43  &  0.46  &  1.00  &  1.00  &  1.00  &  0.86  &  0.86  &  0.86 \\
     V20-10K    &  0.55  &  0.86  &  0.67  &  0.40  &  0.29  &  0.33  &  1.00  &  0.86  &  0.92  &  1.00  &  0.86  &  0.92 \\
    \hline
     V40-2K    &  0.50  &  0.86  &  0.63  &  0.50  &  0.43  &  0.46  &  0.83  &  0.71  &  0.77  &  1.00  &  0.52  &  0.73 \\
     V40-5K    &  0.57  &  1.00  &  0.74  &  0.57  &  0.57  &  0.57  &  1.00  &  1.00  &  1.00  &  1.00  &  1.00  &  1.00 \\
     V40-10K    &  0.41  &  1.00  &  0.58  &  0.30  &  0.43  &  0.35  &  0.88  &  1.00  &  0.93  &  1.00  &  1.00  &  1.00 \\
    \hline
     V60-2K   &  0.27  &  0.57  &  0.36  &  0.00  &  0.00  &  0.00  &  0.80  &  0.57  &  0.67  &  1.00  &  0.57  &  0.73 \\
     V60-5K    &  0.38  &  0.86  &  0.52  &  0.40  &  0.29  &  0.33  &  0.86  &  0.86  &  0.86  &  1.00  &  0.86  &  0.92 \\
     V60-10K   &  0.30  &  0.86  &  0.44  &  0.33  &  0.57  &  0.42  &  0.86  &  0.86  &  0.86  &  0.83  &  0.71  &  0.77 \\
    \hline
     V80-2K   &  0.75  &  0.75  &  0.75  &  1.00  &  0.38  &  0.55  &  1.00  &  0.75  &  0.86  &  1.00  &  0.63  &  0.77 \\
     V80-5K  &  0.55  &  0.75  &  0.63  &  1.00  &  0.50  &  0.67  &  1.00  &  0.75  &  0.86  &  1.00  &  0.75  &  0.86 \\
     V80-10K  &  0.66  &  0.88  &  0.74  &  0.67  &  0.25  &  0.36  &  0.88  &  0.88  &  0.88  &  1.00  &  0.88  &  0.93 \\
    \hline
     V100-2K   &  0.43  &  1.00  &  0.60  &  0.25  &  0.33  &  0.29  &  0.75  &  1.00  &  0.86  &  0.80  &  0.67  &  0.73 \\
     V100-5K  &  0.29  &  0.83  &  0.44  &  0.17  &  0.17  &  0.17  &  0.63  &  0.83  &  0.71  &  0.80  &  0.67  &  0.73 \\
     V100-10K  &  0.35  &  1.00  &  0.52  &  0.57  &  0.67  &  0.62  &  1.00  &  0.83  &  0.91  &  0.71  &  0.83  &  0.77 \\
    \hline
    \end{tabular}%
}
\end{center}
\end{table}%

\begin{table}[h]
  \centering
  \caption{ Wilcoxon signed ranks test results for the four methods with $F_1$ measure listed in Table \ref{tab_numberCAR-CCC-CCU_syn}}
    \begin{tabular}{|c|cccc|}
    \hline
        { \emph{p-value}} &  CR-CS &  PC-select &  CCC   &  CCU \\
    \hline
         CR-CS &  -     &  0.769 &  \textbf{3.74E-04} &  \textbf{2.51E-06} \\
         PC-select &  0.244 &  -     &  \textbf{2.49E-05} &  \textbf{2.63E-06} \\
         CCC   &  1.000 &  1.000 &  -     &  \textbf{0.002} \\
         CCU   &  1.000 &  1.000 &  0.998 &  - \\
    \hline
    \end{tabular}%
  \label{tab_Wil}%
\end{table}%

{To evaluate the ability of CR-CS in recovering combined causal rules, we use synthetic data sets with known combined rules as described in section ~\ref{subsec_Datasets}. We applied CR-CS to the four data sets, V8-2K, V12-2K, V16-2K, and V20-2K-cmb to discover level 2 rules with the 99\% confidence level. The experiment results have shown that CR-CS can recover all known combined rules, including 2 rules in V8-2K, 3 rules in V12-2K, 3 rules in V16-2K, and 4 rules in V20-2K. There are also 5, 3, 16, and 15 extra combined rules discovered by CR-CS in the four data sets, respectively. We do not have  a means to test if extras are real combined causes. However, the results show that the method is able to uncover known combined causes.}

\iffalse
\begin{table}[h]
  \centering
  \caption{ Performance of CR-CS in recovering combined causal rules.}
    \begin{tabular}{|c|cc|cc|}
    \hline
        \multirow{2}[0]{*}{} & \multicolumn{2}{c}{ Single} & \multicolumn{2}{|c|}{ Combined} \\
        &  Precision &  Recall &  Hit Ratio &  Extra \\
    \hline
         V8-2K &  2/2   &  2/2   &  2/2   &  5 \\
         V12-2K &  3/3   &  3/4   &  3/3   &  3 \\
         V16-2K &  3/3   &  3/4   &  3/3   &  16 \\
         V20-2K &  3/3   &  3/4   &  4/4   &  15 \\
    \hline
    \end{tabular}%
  \label{tab_combined}%
\end{table}%

\fi

\subsubsection{Experiment results of real world data}

With the Adult data set, as shown in Table \ref{tab_numberCAR-CCC-CCU}, CR-CS, CCC and CCU discovered similar number of rules, while PC-select found a relatively small number of rules.
\begin{table}[h]
\begin{center}
\caption{Number of causal rules/relationships discovered by CR-CS, CCC, CCU and PC-select with real world data sets}
\label{tab_numberCAR-CCC-CCU}
\begin{tabular}{|l|cccc|}
\hline
 & CR-CS & CCC & CCU &PC-select  \\
\hline
Adult &   46 & 53 &  46 &  19\\
Hypothyroid &  30 & 14 & 10 &4 \\
Sick &  21 & 13 & 3 &5 \\
\hline
\end{tabular}
\end{center}
\end{table}

\begin{table}[h]
\begin{center}
\caption{The similar and dissimilar causal rule groups discovered by CR-CS and the other methods in the Adult data set. {(some-college: Some college but no degree;  exec-managerial: Executive admin and managerial; prof-specialty: Professional specialty; handlers-cleaners: Handlers equip cleaners etc.; machine-op-inspct: Machine operators assemblers \& inspectors; adm-clerical: Admin support including clerical; other-service: Other services; farming-fishing: Farming forestry and fishing. sel-emp-inc: Self-employed-incorporated; sel-emp-not-inc: Self-employed-not incorporated.)}}
\label{tab_causalRuleCensus}
\begin{tabular}{|l|cccc|}
\hline
Causal rules & CR-CS & CCC & CCU &PC-select \\
\hline
Education=doctorate $\to$ $>50$K & $\surd$ & $\surd$ & $\surd$ & $\surd$\\
Education=masters  $\to$ $>50$K & $\surd$ & $\surd$ & $\surd$ & $\surd$\\
Education=bachelors  $\to$ $>50$K & $\surd$ & $\surd$ & $\surd$ & $\surd$\\
Education=prof-School  $\to$ $>50$K & $\surd$ & $\surd$ & $\surd$ & $\surd$\\
Education=some-college $\to$ $\le50$K & & $\surd$ & $\surd$ &\\
Education=HS-grad $\to$ $\le50$K & $\surd$ & $\surd$ & $\surd$ &\\
Education=12th $\to$ $\le50$K & $\surd$ & $\surd$ & $\surd$ &\\
Education=11th $\to$ $\le50$K & $\surd$ & $\surd$ & $\surd$ & $\surd$\\
Education=10th $\to$ $\le50$K & $\surd$ & $\surd$ & $\surd$ & $\surd$\\
Education=9th $\to$ $\le50$K & $\surd$ & $\surd$ & $\surd$ & $\surd$\\
Education=7-8th $\to$ $\le50$K & $\surd$ & $\surd$ & $\surd$ & $\surd$\\
Education=5-6th $\to$ $\le50$K & $\surd$ & $\surd$ & $\surd$ &\\
Education=1-4th $\to$ $\le50$K &  & $\surd$ & $\surd$ &\\
Education=preschool $\to$ $\le50$K &  & $\surd$ &  &\\
\hline
Occupation=exec-managerial $\to$ $>50$K & $\surd$ &  &  & $\surd$\\
Occupation=prof-specialty $\to$ $>50$K & $\surd$ &  &  &\\
%Occupation=protective serv $\to$ $>50$K & $\surd$ &  &  &\\
Occupation=tech-support $\to$ $>50$K & $\surd$ & $\surd$ & $\surd$ &\\
Occupation=sales $\to$ $>50$K & $\surd$ &  &  &\\
Occupation=handlers-cleaners $\to$ $\le50$K & $\surd$ &  &  & $\surd$\\
Occupation=machine-op-inspct $\to$ $\le50$K & $\surd$ &  & & \\
Occupation=adm-clerical $\to$ $\le50$K & $\surd$ &  &  &\\
Occupation=other-service $\to$ $\le50$K & $\surd$ &  &  & $\surd$\\
Occupation=farming-fishing $\to$ $\le50$K & $\surd$ &  & & $\surd$ \\
Occupation=transport-moving $\to$ $\le50$K & $\surd$ &  & & \\
Occupation=craft-repair $\to$ $\le50$K & $\surd$ &  &  &\\
\hline
Workclass=sel-emp-inc $\to$ $>50$K & $\surd$ & $\surd$ &  & $\surd$\\
Workclass=sel-emp-not-inc $\to$ $>50$K & & $\surd$ & $\surd$ &\\
Workclass=federal-gov $\to$ $>50$K & $\surd$ & $\surd$ & $\surd$ & $\surd$\\
Workclass=state-gov $\to$ $>50$K & $\surd$ &  &  &\\
Workclass=local-gov $\to$ $>50$K & $\surd$ & $\surd$ & $\surd$ &\\
Workclass=private $\to$ $\le50$K & & $\surd$ & $\surd$ &\\
%\hline
%Sex=male $\to$ $>50$K & $\surd$ & $\surd$ & $\surd$ \\
%sex=female $\to$ $\le50$K & $\surd$ &  &  \\
\hline
Native Country=USA $>50$K & $\surd$ & $\surd$ & $\surd$ &\\
Native Country=various countries &  1 & 22 & 17 & 2\\
%\hline
\hline %\hline
% Occupation=Sales \&   Sex=Male $\to$ $>50$K &  $\surd$ &  &  & \\
 Education=Some-college  &  &  &  & \\
 \& Workclass=Private $\to$ $\le 50$K &  $\surd$ &  &  & \\
\hline
\end{tabular}
\end{center}
\end{table}

When we look into the rules discovered by these methods, they are quite different. We list in Table \ref{tab_causalRuleCensus} the most similar and dissimilar rule groups found in the Adult data set using CR-CS and the other methods. We can see that overall CR-CS and PC-select obtained similar results, while only for the variables related to the Education attribute, rules discovered by CR-CS and PC-select are similar to those discovered by CCC and CCU.

Intuitively, Education is the major factor affecting incomes. We see that people with higher education have a better chance for a high salary, such as, doctorate, masters, bachelors, and professional school (prof-School). In contrast, people with lower education more likely receive a low salary, for example some college but no degree and lower.
%The effects of Workclass on salary are also intuitively right. Rules discovered by the four methods on the variables are basically consistent.

Rules discovered by CR-CS and PC-select are dissimilar to those found by CCC and CCU in relation to the Occupation, Workclass and Native-country attributes. There are 11 rules discovered by CR-CS with respect to the Occupation attributes, but only one rule is discovered by CCC and CCU in this group. CCC and CCU have missed some very reasonable causal factors for high/low salary. For example, ``exec-managerial" and ``prof-specialty'' for high salary, and ``handlers-cleansers'' and ``adm-clerical'' for low salary are reasonable causal rules, but they have been missed by CCC and CCU. PC-select, although found fewer number of rules in this group (Occupation), the rules found by it are reasonable. On the other hand, 22 rules related to the Native Country attributes are discovered by CCC, 17 rules by CCU, but only 1 rule by CR-CS in this group. PC-select found two rules in this group, again performing more consistently with CR-CS. Intuitively, Native Country should not a factor for high/low salary. This shows that CR-CS is able to discover more reasonable causal rules.

{ 
The combined causal rule discovered by CR-CS is also reasonable. As shown in Table \ref{tab_causalRuleCensus}, people with some-college education but without any degree and working in a private sector would have low salaries. CR-CS did not discover that people with some-college or with private work-class would have low income at the single rule level as found by CCC and CCU, but it provides more details with the combined causal rule.

To investigate the number of combined causal rules in real world cases, we run CR-CS for eight real world data sets with up to level 4 rules. Table \ref{tab_number}  shows the number of single and combined causal rules discovered by CR-CS with the 95\% confidence level. We can see that the combined causal rules at level 3 and 4 are rare, but CR-CS found a number of combined causal rules at the second level. Although we do not have a ground truth to validate all of the combined rules, some rules are reasonable based on common knowledge. For example, with the Mushroom data set, we can see from Table ~\ref{tab_CombinedEx} that poisonous mushrooms are pink and have either  evanescent ring type or white spore print. Our common understanding is that poisonous mushrooms are normally in bright color, but not all brightly colored mushrooms are poisonous. These combined causal rules provide more detail on the poisonous mushrooms than just based on their colors, and therefore they are useful in practice. Similarly, CR-CS discovers that mushrooms without bright color and odor are edible, and these rules are also reasonable.}

\begin{table}[h]
  \centering
  \caption{  Number of combined causal rules discovered by CR-CS in real world data sets}
     \begin{tabular}{|c|cccc|}
    \hline
          &  1st Level &  2nd Level &  3rd Level &  4th Level \\
        \hline
     { Adult}     &  45    &  1     &  0     &  0 \\ \hline
    % { BCW}     &  51    &  0     &  0     &  0 \\ \hline
     { Census income}     &  77    &  6     &  0     &  0 \\ \hline
     { Germand}     &  8    &  38     &  12     &  5 \\ \hline
     { Hypothyroid}     &  20    &  7     &  3     &  0 \\ \hline
     { Kr-vs-kp}     &  3    &  15     &  0     &  0 \\ \hline
     { Mushrom}     &  26    &  61     &  0     &  1 \\ \hline
     { Sick}     &  13    &  7     &  1     &  0 \\ \hline
     { Tic-tac}     &  8    &  30     &  3     &  0 \\ \hline
%    { Vote}      &  0    &  0     &  0     &  0 \\ \hline
    \end{tabular}%
  \label{tab_number}%
\end{table}%

\begin{table}[h]
  \centering
  \caption{  Some combined causal rules in the Mushrooms data set.}
     \begin{tabular}{|c|}
    \hline
              Combined causal rules \\
        \hline
    {  Stalk-color-below-ring = pink \& Ring-type = evanescent  $\to$ poisoners  }\\ \hline
  {  Stalk-color-below-ring = pink \& Spore-print-color = white $\to$ poisoners }\\ \hline
    { Odor = none \& Stalk-shape = tapering $\to$ edible}\\ \hline
    {  Cap-color = gray \& Odor = none $\to$ edible}\\ \hline
    \end{tabular}%
  \label{tab_CombinedEx}%
\end{table}%

\subsection{Stability}\label{sectionStability} \label{sec_samplebiase}

The creation of a fair data set is subject to selection bias. Usually there are significantly more exposed cases than non-exposed cases so the data distribution is often skewed for the exposure and non-exposure conditions.  When we choose pairs of matched records to form a fair data set, we pick up one record from the exposure group and find a matched record from the non-exposure group. In this process, the values of the response variable are blinded. When there are more than one matched record to choose from, we randomly choose one. It is possible that the value distribution of the response variable in a fair data set is affected by the random selection. This will cause misses or false discoveries of causal rules. This situation is the same as the real world sample process, which is subjected to sampling bias.

To reduce the impact of selection bias, we run the method on a data set multiple times and select consistent rules in multiple causal rule sets as the final causal rules. The variance is not big and the causal discovery is quite stable. The numbers of causal rules from different runs and the rules supported by two causal rule sets are listed in Table \ref{tab_combinedRules}. On a large data set, such as the Adult data set, the change of rules between different runs is very small. Only one rule difference in the three runs. Even in a small data set, such as the Sick data set, nearly 90\% rules are consistent over three runs.

\begin{table}[h]
\begin{center}
\caption{The numbers of causal rules of different runs and the frequent causal rules }
\label{tab_combinedRules}
\begin{tabular}{|l|c|c|c|c|}
\hline
 fair data set &  1 & 2 & 3 & frequent \\
\hline
Adult &  46 &  46 &  45 &  46 \\
Hypothyroid &  31 &  30 &  30 &  30\\
Sick &  21 &  20 &  21 &  21 \\
\hline
\end{tabular}
\end{center}
\end{table}

\subsection{Results obtained using different matching methods}\label{sectionMatching}

As described in Section 3.4 (Definition 3.4), when creating a fair data set, different similarity measures can be used for finding matched pairs of records. In the experiments described so far, exact matching has been used. In order to gain some insights into the impact of different similarity measures, we also experimented on our method when Jaccard distance is used in matching a pair of records. Jaccard distance \cite{DataMiningBookHan2005} is a commonly used measure of the similarity between records with binary attributes. From Table \ref{tab_matching}, we see that the numbers of rules discovered are very similar across the three data sets with exact matching and the matching using Jaccard distance.

\begin{table}[h]
\begin{center}
\caption{Results of CR-CS using different matching methods}
\label{tab_matching}
\begin{tabular}{|l|c|c|c|c|}
\hline
Data set &  Exact matching & Jaccard distance \\
\hline
Adult &  46 &  46  \\
Hypothyroid &  30 &  31 \\
Sick &  21 &  22  \\
\hline
\end{tabular}
\end{center}
\end{table}

\subsection{Efficiency}\label{sectionEfficiency}

To test the time efficiency of CR-CS, we applied it to the Census Income (KDD) data set and  the last five synthetic data sets (with 10K records), to observe its scalability in terms of the number of records and the number of attributes respectively. The experiments were also done in comparison with the other three methods.

As the original CCC and CCU algorithms do not assume a fixed response variable, we ran them with the restriction of only looking for the triplets that contain the response variable. For our method, we ran it in two different versions: CR-CS1 and CR-CS2 respectively. With CR-CS1, we constrained the length of rules to 1, making it comparable with CCC, CCU and PC-select. With CR-CS2, the length of rules was restricted to 2 to allow the discovery of combined causes. CR-CS1 and CR-CS2 were implemented in Java, CCC and CCU were implemented in Matlab, and for PC-select, we used the pcSelect() function of the R package \emph{pcalg} \cite{pcalg,Kalisch:2012}. The comparisons were carried out using the same desktop computer (Quad core CPU 3.4 GHz and 16 GB of memory).

The execution time (in seconds) of  CR-CS1, CR-CS2, CCC and CCU with respect to the number of records in the Census Income (KDD) data is shown in Figure \ref{fig_scalability_records}. The execution time of PC-select is not included as it did not return results on any data set after two hours of execution. From the figure, we can see that CR-CS1 was much faster than CCC and CCU consistently for different record sizes, and even CR-CS2 was also faster than the other methods. The main reason is that our method employs association rule mining to remove non-eligible rules and thus to reduce the search space significantly.

The  execution time of CR-CS1, CR-CS2, CCC, CCU and PC-select with respect to the number of attributes is shown in Figure \ref{fig_scalability_attributes} (only the results returned within 6 hours are shown). Similarly, CR-CS1 is more scalable than CCC and CCU, while CR-CS2 is much slower when the number of attributes became big as the number of association rules increased significantly with the increase of the number of attributes, leading to additional time for testing causal rules. Although PC-select can achieve high quality of causal discovery (see Table \ref{tab_numberCAR-CCC-CCU_syn}), from the Figure \ref{fig_scalability_attributes}, we can see that PC-select is inefficient or even infeasible, especially when the number of variables is large.

\begin{figure}[t]
\begin{center}
\includegraphics[width=0.65\textwidth]{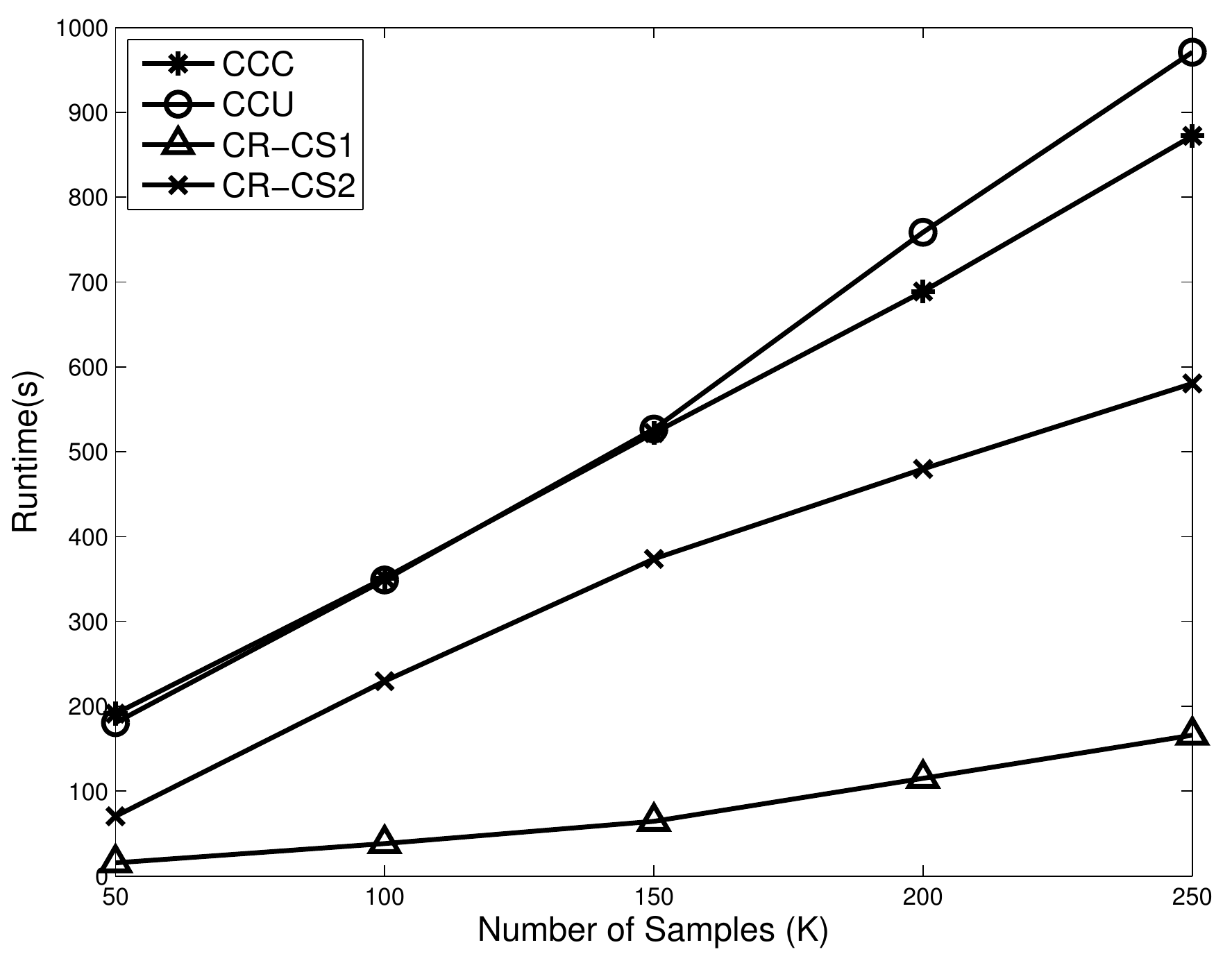}
\caption{Scalability with respect to number of records (note:PC-select is not included since it did not return results after two hours of execution)}
\label{fig_scalability_records}
\end{center}
\end{figure}

\begin{figure}[h]
\begin{center}
\includegraphics[width=0.65\textwidth]{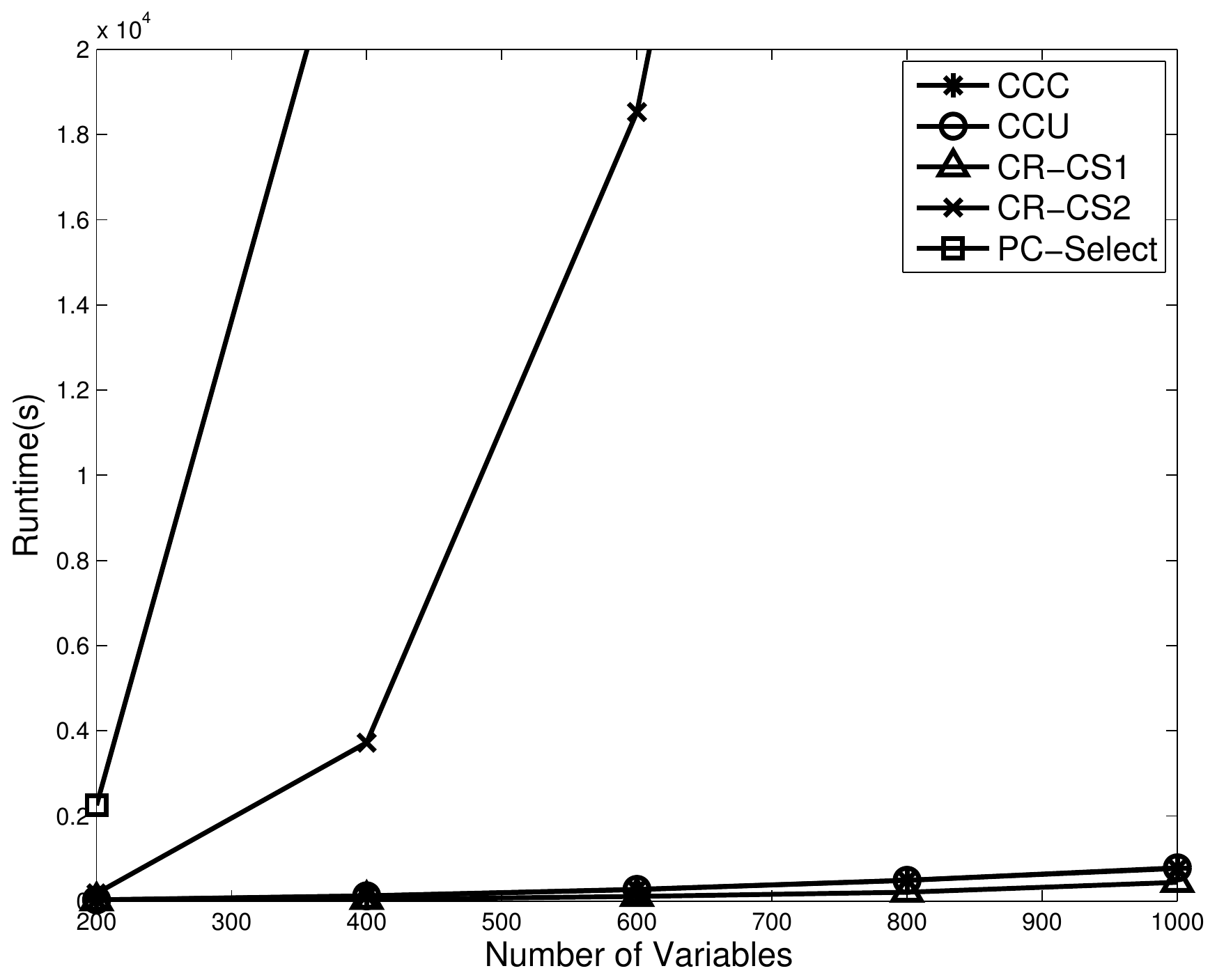}
\caption{Scalability with respect to number of attributes}
\label{fig_scalability_attributes}
\end{center}
\end{figure}

\section{Conclusion and future work}
In this paper, we have proposed  the concept of causal rules and have developed a method to find causal rules from observational data by integrating association rule mining with retrospective cohort studies. Through the integration, our method has been able to take the advantage of the high efficiency of association rule mining to produce candidate causal relationships from large data sets, and then to utilise the idea of cohort studies to obtain reliable causal rules based on the candidates.  The validity of the definition of causal rules has been justified to be consistent with the potential outcome model. Experiments results have shown that the proposed method is able to find more reasonable causal relationships comparing to the existing causal discovery methods. Moreover, our method was able to find causes consisting of combined variables, which are not possible to be uncovered by the other existing methods. We have shown that the method is faster than the efficient constraint based causal relationship discovery methods. Hence our method can be used as a promising alternative for causal discovery in large and high dimensional data sets. With the proposed method, the selection of control variable set is a key to discovering quality causal rules. The validation of the control variable set in real world applications will ensure the quality of causal rules discovered.

The proposed causal rule mining method and the constraint based causal discovery approaches tackle the problem of causal discovery from different directions. They each have their own strengths and limitations. Our future work will be studying how they complement each other and exploring integrated methods for efficient and quality causal relationship discovery.

\begin{acks}
This work has been partially supported by Australian Research Council Discovery Project DP130104090 and DP140103617.
\end{acks}

% Bibliography
\bibliographystyle{ACM-Reference-Format-Journals}
\bibliography{CR-CS}
                             % Sample .bib file with references that match those in
                             % the 'Specifications Document (V1.5)' as well containing
                             % 'legacy' bibs and bibs with 'alternate codings'.
                             % Gerry Murray - March 2012

% History dates
\received{xx 2014}{xx 2014}{xx 2015}

\end{document}